\PassOptionsToPackage{dvipsnames}{xcolor}
\documentclass[final,12pt]{colt2026}
\usepackage[utf8]{inputenc}
\usepackage[T1]{fontenc}
\usepackage{hyperref}
\usepackage{subfiles}

\usepackage{textcomp}        % Text companion fonts (extra symbols)
\usepackage{cancel}          % Strikeout
\usepackage[normalem]{ulem}  % Underlining
\usepackage{dsfont}          % Double-stroke fonts
\usepackage{multirow}        % Allows table cells to span multiple rows
\usepackage{makecell}        % Makecell inside a table

\usepackage{thmtools,thm-restate}
\usepackage{mathtools}

\usepackage{pifont}
\newcommand{\cxmark}{{\color{black}\ding{51}\kern-0.642em\ding{55}}}

\newtheorem{myRemark}{Remark}

\newtheorem{manualtheoreminner}{Theorem}
\newenvironment{manualtheorem}[1]{%
  \renewcommand\themanualtheoreminner{#1}%
  \manualtheoreminner
}{\endmanualtheoreminner}

\usepackage{prettyref}
\newcommand{\pref}[1]{\prettyref{#1}}

\newcommand{\savehyperref}[2]{\texorpdfstring{\hyperref[#1]{#2}}{#2}}
\newrefformat{eq}{\savehyperref{#1}{Eq.~\textup{(\ref*{#1})}}}
\newrefformat{eqn}{\savehyperref{#1}{Eq.~(\ref*{#1})}}
\newrefformat{lem}{\savehyperref{#1}{Lemma~\ref*{#1}}}
\newrefformat{lemma}{\savehyperref{#1}{Lemma~\ref*{#1}}}
\newrefformat{def}{\savehyperref{#1}{Definition~\ref*{#1}}}
\newrefformat{line}{\savehyperref{#1}{Line~\ref*{#1}}}
\newrefformat{thm}{\savehyperref{#1}{Theorem~\ref*{#1}}}
\newrefformat{table}{\savehyperref{#1}{Table~\ref*{#1}}}
\newrefformat{corr}{\savehyperref{#1}{Corollary~\ref*{#1}}}
\newrefformat{cor}{\savehyperref{#1}{Corollary~\ref*{#1}}}
\newrefformat{sec}{\savehyperref{#1}{Section~\ref*{#1}}}
\newrefformat{subsec}{\savehyperref{#1}{Section~\ref*{#1}}}
\newrefformat{app}{\savehyperref{#1}{Appendix~\ref*{#1}}}
\newrefformat{appendix}{\savehyperref{#1}{Appendix~\ref*{#1}}}
\newrefformat{assum}{\savehyperref{#1}{Assumption~\ref*{#1}}}
\newrefformat{ex}{\savehyperref{#1}{Example~\ref*{#1}}}
\newrefformat{fig}{\savehyperref{#1}{Figure~\ref*{#1}}}
\newrefformat{alg}{\savehyperref{#1}{Algorithm~\ref*{#1}}}
\newrefformat{remark}{\savehyperref{#1}{Remark~\ref*{#1}}}
\newrefformat{conj}{\savehyperref{#1}{Conjecture~\ref*{#1}}}
\newrefformat{prop}{\savehyperref{#1}{Proposition~\ref*{#1}}}
\newrefformat{proto}{\savehyperref{#1}{Protocol~\ref*{#1}}}
\newrefformat{prob}{\savehyperref{#1}{Problem~\ref*{#1}}}
\newrefformat{claim}{\savehyperref{#1}{Claim~\ref*{#1}}}
\newrefformat{que}{\savehyperref{#1}{Question~\ref*{#1}}}
\newrefformat{op}{\savehyperref{#1}{Open Problem~\ref*{#1}}}
\newrefformat{fn}{\savehyperref{#1}{Footnote~\ref*{#1}}}
\newrefformat{require}{\savehyperref{#1}{Requirement~\ref*{#1}}}
\newrefformat{row}{\savehyperref{#1}{Row~\ref*{#1}}}

\newcommand{\KwInput}[1]{\textbf{Input}: #1}
\newcommand{\KwInitialize}[1]{\textbf{Initialize}: #1}

\usepackage{tcolorbox}

\newcounter{BoxCounter}
\newcounter{PreBoxCounter}
\newenvironment{Boxed*}[2][\small]
{
  \begin{figure*}[!t]
    \begin{tcolorbox}[title=\textbf{Box \arabic{PreBoxCounter}:} #2, colback=white, colbacktitle=white, coltitle=black, arc=0pt,outer arc=0pt, fontupper=#1]
    }{
    \end{tcolorbox}
  \end{figure*}
}

\newcommand{\half}{\frac{1}{2}}

\newcommand{\grad}{\nabla}
\newcommand{\eps}{\varepsilon}
\newcommand{\sumtT}{\sum_{t=1}^T}

\usepackage{tikz}
\newcommand\encircle[1]{%
  \tikz[baseline=(X.base)]
  \node (X) [draw, shape=circle, inner sep=0] {\strut #1};}

\renewcommand*\hat\widehat
\renewcommand*\tilde\widetilde

\newcommand{\tpp}{{t+1}}
\newcommand{\tmm}{{t-1}}

\newcommand{\bbE}{\mathbb{E}}

\newcommand{\bbN}{\mathbb{N}}

\newcommand{\bbR}{\mathbb{R}}

\newcommand{\cA}{\mathcal{A}}
\newcommand{\cB}{\mathcal{B}}

\newcommand{\cD}{\mathcal{D}}

\newcommand{\cF}{\mathcal{F}}

\newcommand{\cI}{\mathcal{I}}

\newcommand{\cU}{\mathcal{U}}

\newcommand{\cW}{\mathcal{W}}

\newcommand{\fD}{\mathfrak{D}}

\newcommand{\ft}{\mathfrak{t}}

\newcommand{\zeros}{\mathbf{0}}

\newcommand{\argmin}{\operatorname{arg\,min}}

\newcommand{\Min}[1]{\min\Set{#1}}

\newcommand{\norm}[1]{\left\|#1\right\|}

\newcommand{\abs}[1]{\left|#1\right|}

\newcommand{\inner}[1]{\left\langle #1 \right\rangle}

\newcommand{\Log}[1]{\log\left(#1\right)}
\newcommand{\Exp}[1]{\exp\left(#1\right)}

\newcommand{\Set}[1]{\left\{#1\right\}}
\newcommand{\sbrac}[1]{\left[#1\right]}
\newcommand{\brac}[1]{\left(#1\right)}

\newcommand{\bbr}[1]{\left\{#1\right\}}
\newcommand{\mbr}[1]{\left[#1\right]}
\newcommand{\sbr}[1]{\left(#1\right)}

\newcommand{\Otilde}{\widetilde{O}}

\newcommand{\R}{\bbR}
\newcommand{\N}{\bbN}
\newcommand{\E}{\bbE}

\usepackage{version}
\makeatletter
\newcommand{\SetVersionColor}[2]{%
  \colorlet{ver@clr@#1}{#2}%
  \expandafter\def\csname vercolor@#1\endcsname{ver@clr@#1}%
}
\newcommand{\VersionColorName}[1]{%
  \@ifundefined{vercolor@#1}{black}{\csname vercolor@#1\endcsname}%
}
\newcommand{\EnableVersion}[1]{%
  \includeversion{#1}%
  \expandafter\def\csname ver@#1\endcsname{1}%
  \expandafter\AtBeginEnvironment\expandafter{#1}{\begingroup\color{\VersionColorName{#1}}}%
  \expandafter\AtEndEnvironment\expandafter{#1}{\endgroup}%
}
\newcommand{\DisableVersion}[1]{\excludeversion{#1}\expandafter\let\csname ver@#1\endcsname\relax}
\newcommand{\IfVersionTF}[3]{\expandafter\ifx\csname ver@#1\endcsname\relax #3\else #2\fi}
\newcommand{\V}[2]{\IfVersionTF{#1}{\textcolor{\VersionColorName{#1}}{#2}}{}}
\makeatother

\usepackage{xparse}
\ExplSyntaxOn

\NewDocumentCommand{\DeclareSubSymbol}{mm}
  {
    \DeclareDocumentCommand #1 { g } % overwrite if already defined
      {
        \ensuremath{#2 \IfNoValueF{##1}{\sb{##1}}}
      }
  }

\NewDocumentCommand{\DeclareTimeShiftVariants}{mm}
  {
    \cs_set:cpn { \cs_to_str:N #1 \tl_to_str:n {#2} }      { #1{#2} }
    \cs_set:cpn { \cs_to_str:N #1 \tl_to_str:n {#2} pp }   { #1{#2+1} }
    \cs_set:cpn { \cs_to_str:N #1 \tl_to_str:n {#2} mm }   { #1{#2-1} }
  }

\NewDocumentCommand{\DeclareShiftedSymbol}{mmO{t}}
  {
    \DeclareSubSymbol{#1}{#2}
    \clist_map_inline:nn {#3} { \DeclareTimeShiftVariants{#1}{##1} }
  }
\ExplSyntaxOff
\DeclareShiftedSymbol{\w}{w}[t,s]
\DeclareShiftedSymbol{\what}{\hat{w}}
\DeclareShiftedSymbol{\wtilde}{\tilde{w}}
\DeclareShiftedSymbol{\y}{y}[t,s]
\DeclareShiftedSymbol{\cmpy}{\mathring{y}}[t,s]
\DeclareShiftedSymbol{\tg}{\tilde{g}}[t,s]
\DeclareShiftedSymbol{\x}{x}[t,s]
\DeclareShiftedSymbol{\f}{f}[t,s]
\DeclareShiftedSymbol{\z}{z}[t,s]
\DeclareShiftedSymbol{\v}{v}[t,s]
\DeclareShiftedSymbol{\yhat}{\hat{y}}
\DeclareShiftedSymbol{\ghat}{\hat{g}}
\DeclareShiftedSymbol{\gtilde}{\tilde{g}}
\DeclareShiftedSymbol{\W}{W}
\DeclareShiftedSymbol{\g}{g}[s,t]
\DeclareShiftedSymbol{\G}{G}
\DeclareShiftedSymbol{\cmp}{u}
\DeclareShiftedSymbol{\Cmp}{U}
\DeclareShiftedSymbol{\ellhat}{\hat{\ell}}
\DeclareShiftedSymbol{\elltilde}{\tilde{\ell}}

\newcommand{\h}{h}

\newcommand{\ww}{\mathcal{W}}
\newcommand{\Vbar}{\bar{V}}

\newcommand{\onedim}{\mathrm{1d}}
\newcommand{\define}{\triangleq}
\newcommand{\reg}{\textsc{Reg}}
\newcommand{\is}{i_\star}

\usepackage{cleveref}

\definecolor{myblue}{RGB}{0,112,192}
\definecolor{myred}{RGB}{192,0,1}
\definecolor{wine_red}{RGB}{228,48,64}
\definecolor{DSgray}{cmyk}{0,1,0,0}

\def \W {\mathcal{W}}

\allowdisplaybreaks
\title[Unconstrained Gradient-Variation Regret]{Gradient-Variation Regret Bounds for Unconstrained Online Learning}
\usepackage{times}

\coltauthor{
 \Name{Yuheng Zhao\nametag{\thanks{Equal Contribution.}}} \Email{zhaoyh@lamda.nju.edu.cn}\\
 \addr State Key Laboratory for Novel Software Technology, Nanjing University\\
 School of Artificial Intelligence, Nanjing University
 \AND
 \Name{Andrew Jacobsen\nametag{\footnotemark[1]}} \Email{contact@andrew-jacobsen.com}\\
 \addr Universit\`{a} degli Studi di Milano \& Politecnico di Milano
 \AND
 \Name{Nicol\`{o} Cesa-Bianchi} \Email{nicolo.cesa-bianchi@unimi.it}\\
 \addr Universit\`{a} degli Studi di Milano \& Politecnico di Milano
  \AND
 \Name{Peng Zhao\nametag{\thanks{Corresponding Author.}}} \Email{zhaop@lamda.nju.edu.cn}\\
 \addr State Key Laboratory for Novel Software Technology, Nanjing University\\
 School of Artificial Intelligence, Nanjing University
 }

\begin{document}

\maketitle

\begin{abstract}
  We develop parameter-free algorithms for unconstrained online learning
  with regret guarantees that scale with the gradient variation $V_T(u) = \sum_{t=2}^T \norm{\grad f_t(\cmp)-\grad f_\tmm(\cmp)}^2$.
  For $L$-smooth convex losses, we provide fully-adaptive  algorithms achieving regret of
  \mbox{$\tilde{O}(\norm{u}\sqrt{V_T(u)} + L\norm{u}^2+G^4)$} without requiring prior knowledge of comparator norm $\norm{u}$, Lipschitz constant $G$, or smoothness $L$. The update in each round can be computed efficiently via a closed-form expression.
  Our results extend to dynamic regret and find immediate implications for the \mbox{stochastically-extended adversarial} (SEA) model, which significantly improves upon the previous \mbox{best-known result \citep{wang2025parameter}.}
\end{abstract}

\section{Introduction}

Online learning~\citep{book'06:PLG-Bianchi,orabona2019modern} is a fundamental paradigm in machine learning for modeling and analyzing sequential prediction and decision-making problems. 
An online learning process is formalized as an interaction between a learner and the environment. At iteration $t \in [T]$, the learner chooses a decision $w_t$ from a feasible domain $\cW\subseteq \R^d$, after which the environment reveals a loss function $f_t:\cW\to\R$, and the learner incurs a loss $f_t(w_t)$. A general performance metric is the \emph{dynamic regret}~\citep{ICML'03:zinkevich,NIPS'18:Zhang-Ader}, which evaluates the cumulative loss against a sequence of comparators:
\begin{equation}
    \label{eq:def-dynamic-regret}
    \reg_T(u_{1:T}) \define \sumtT f_t(w_t) - \sumtT f_t(u_t),
\end{equation}
where the comparators $\cmp_{1:T} \triangleq (u_1, \ldots, u_T)$ in $\cW$ are unknown, and their variability is typically measured by the path length $P_T(u_{1:T}) \define \sum_{t=2}^T \norm{u_t - u_{t-1}}$. When restricting to a fixed comparator $u\in\cW$,  dynamic regret reduces to the standard notion of static regret, denoted by $\reg_T(u)$.

\subsection{Gradient-Variation Adaptivity}
\label{subsec:gradient-variation-adaptivity}
A rich theory has been developed for both static and dynamic regret minimization in the last decades \citep{ICML'03:zinkevich,book'12:Shai-OCO,orabona2019modern,hazan2022OCObook}. 
Notably, recent studies suggest the importance of \emph{problem-dependent adaptivity}~\citep{JMLR'14:FlipFlop, NIPS'15:Dylan-adaptive,book2021:Roughgarden,JMLR'24:Sword++}, which aims to achieve tighter bounds for benign problem instances while preserving minimax optimality in the worst case. 
Among various measures of problem difficulty, a key quantity is the \emph{gradient variation}~\citep{COLT'12:VT, ML'14:variation-Yang}, defined as
\begin{equation}
    \label{eq:def-gradient-variation}
    V_T^+ \define \sum_{t=2}^{T} \sup_{w\in\cW} \norm{\nabla f_t(w) - \nabla f_{t-1}(w)}^2,
\end{equation}
which captures how the problem evolves over time in terms of the function gradients. When the feasible domain is bounded (i.e., $\norm{x - y}\le D$ for all $x,y\in\cW$) and the online functions are convex and $L$-smooth (i.e., $\norm{\grad f_t(x) - \grad f_t(y)}\le L\norm{x - y}$ for all $x,y\in\cW$), it is known that optimal $\smash{O\big(D\sqrt{V_T^{\smash{+}}} + L D^2\big)}$ static regret~\citep{COLT'12:VT} and $\smash{O\big(D\sqrt{V_T^{\smash{+}} (1+P_T)} + L(D^2+DP_T)\big)}$ dynamic regret~\citep{NeurIPS'20:sword} can be achieved.
Gradient-variation-based online learning has attracted growing interest in recent years. In particular,~\citet{NeurIPS'20:sword} introduced gradient variation into dynamic regret minimization and proposed novel techniques that have inspired many subsequent works
\citep{ICML'22:TVgame,NeurIPS'22:SEA,NeurIPS'24:LocalSmooth,ICML'25:FTRL-optimistic,NeurIPS'25:GV4OPT,wang2025parameter,arxiv'25:Universal-Zhao,ICML'26:GV-bandits}.
Gradient-variation adaptivity has been revealed to have tight connections to a broad class of optimization problems. 
For example, this adaptivity is shown to be crucial for achieving fast convergence rates in minimax optimization/games~\citep{NIPS'15:fast-rate-game, ICML'22:TVgame}, as well as for attaining acceleration in offline smooth convex optimization~\citep{ICML'19:Ashok-acceleration, NeurIPS'25:GV4OPT}.
Moreover, recent work demonstrates that controlling gradient-variation regret is essential for obtaining adaptive guarantees under the \emph{Stochastically Extended Adversarial} (SEA) model, which interpolates between adversarial online optimization and stochastic convex optimization~\citep{NeurIPS'22:SEA,JMLR'24:OMD4SEA}.

\subsection{Parameter-Free Online Learning}
Most existing gradient-variation online learning algorithms rely on the assumption of a bounded feasible domain. 
In many practical scenarios, however, the domain is naturally \emph{unbounded}, making it infeasible to impose an \emph{a priori} diameter upper bound $D$ on the comparator norm $\norm{u}$.
This limitation motivates the study of \emph{parameter-free online learning}~\citep{NIPS09NormalHedge,mcmahan2012noregret,NIPS'13:Orabona-dimension,mcmahan2014unconstrained,NIPS'16:coin-betting-Orabona,jacobsen2022parameter,cutkosky2024fully}, which aims to design algorithms that do not require such problem-dependent quantities as inputs.

A central requirement of parameter-free algorithms is being \emph{comparator-adaptive}: the ability to achieve regret bounds that scale favorably with the unknown comparator norm $\norm{u}$.
When the Lipschitz constant $G$ of loss functions is further unknown,\footnote{The Lipschitz constant $G$ is used to denote the empirical gradient norm $\max_t\norm{\nabla f_t(w_t)}$, when there's no ambiguity.} a parameter-free online algorithm must also be \emph{Lipschitz-adaptive}, meaning it attains the desired regret without prior knowledge of $G$.
Algorithms that satisfy both comparator-adaptivity and Lipschitz-adaptivity are sometimes referred to as \emph{fully-adaptive} methods~\citep{cutkosky2024fully}. 
The best known fully-adaptive result is achieved by~\citet{cutkosky2024fully}, who attained  a static regret of order $\reg_T(u) \leq \tilde{O}\big( \norm{u}G\sqrt{T} + \norm{u}^2 + G^2 \big)$ with $\tilde{O}(\cdot)$ omitting poly-logarithmic factors.

For gradient-variation regret over unbounded domains, two previous works are most relevant.
\citet{jacobsen2022parameter} proposes a mirror descent-based algorithm that obtains comparator-adaptive gradient-variation regret, but requires full-information feedback $f_t(\cdot)$ since it applies 
an \emph{implicit} update \citep{campolongo2020temporal} in its optimistic step.
\citet{wang2025parameter}
study the stochastically-extended adversarial (SEA) model---a closely related setting where gradient-variation online learning plays a central role---and
attain a comparator-adaptive regret bound of  $\smash{\tilde{O}\big(\norm{u}\sqrt{V_T^{\smash{+}}} + \norm{u}^2\big)}$ when the Lipschitz constant is known, using only first-order feedback. 
However, their algorithm relies on a two-layer meta-base structure that maintains $\smash{O(\log^2 T)}$ base learners, which is significantly more expensive than the $O(d)$ per-round computation used to obtain gradient-variation bounds in the bounded domain setting.
Moreover, when further targeting Lipschitz adaptivity, their method suffers a significant deterioration of the leading term to $\norm{u}^2\sqrt{V_T^{\smash{+}}}$, giving a sub-optimal dependence on the comparator norm.
As such, a natural open question arises:
\begin{center}
\emph{Is it possible to achieve gradient-variation regret in a fully adaptive, parameter-free manner over unbounded domains with an efficient algorithm?}\vspace{-1mm}
\end{center}

\subsection{Our Contributions}

In this paper, we provide an affirmative answer by developing the \emph{first} fully-adaptive algorithm for gradient-variation online learning in the unconstrained domain, requiring no prior knowledge of the comparator norm $\norm{u}$, the Lipschitz constant $G$, or the smoothness parameter $L$. Notably, our algorithm is efficient, with a closed-form update that can be computed in $O(d)$ time per round. 

To begin, we clarify the definition of gradient variation in the context of unbounded domains. The original definition of $V_T^{\smash{+}}$ in~\pref{eq:def-gradient-variation} does not work well in this context, as the $\R^d$-domain may easily cause it to scale as $O(G^2T)$, reducing back to a non-adaptive worst-case dependence. 
Instead, we introduce a more appropriate definition that supports arbitrary comparators $u_{1:T}\in\R^d$ as:
\begin{equation}
    \label{eq:def-gradient-variation-u-t}
  V_T(u_{1:T}) \define \sum_{t=2}^T \norm{\grad f_{t}(u_{t-1})-\grad f_{t-1}(u_{t-1})}^{2},
\end{equation}
which quantifies the gradient variation between consecutive functions on the sequence. The time-varying comparators primarily serve to accommodate dynamic regret.
For static regret over a bounded domain $\cW$, it captures the original definition in~\pref{eq:def-gradient-variation} since $V_T(u)\le V_T^{\smash{+}}$ for any $u\in\cW$.

We provide both comparator-adaptive and fully-adaptive gradient-variation regret bounds. A summary of our contributions and a comparison to prior works can be found in~\pref{table:results}.

\begin{table}[!t]
    \centering
    \caption{\small{Comparison of parameter-free gradient-adaptive regret, where  
    $\gt\in\partial\ft(\wt)$ denotes the gradient feedback on round $t$.
    ``Efficiency'' denotes the order of per-round computational cost. 
    }
    }
    \vspace{-3mm}
    \label{table:results}
    \renewcommand{\arraystretch}{1.4}
    \resizebox{0.99\textwidth}{!}{
    \begin{tabular}{c|c|c|c}
    \hline

    \hline
    {\textbf{Setting}} & {\textbf{Reference}} & {\textbf{Regret} in $\tilde{O}(\cdot)$-notation}  & \textbf{Efficiency} \\ \hline
    
    \rule{0mm}{6mm}
    \multirow{3.5}{*}{\makecell{\textbf{Comparator}\\ \textbf{Adaptive}}} & \citet{jacobsen2022parameter} & $\norm{u}\sqrt{\sum_{t=2}^T\norm{g_t - g_{t-1}}^2} + G\norm{u}$   & $O(d)$\\[1mm] \cline{2-4}
    
    \rule{0mm}{6mm}
    & \citet{wang2025parameter} & {$\norm{u}\sqrt{V_T^{\smash{+}}} + L^2\norm{u}^2/G^2 + G^2\norm{u}^2$}  & $O(d\log^2 T)$\\[1mm] \cline{2-4}

    \rule{0mm}{6mm}
    & \textbf{Ours},~\pref{thm:warmup} & {$\norm{u}\sqrt{V_T(u)} + L\norm{u}^2 + G\norm{u}$}  & $O(d)$\\[1mm] \hline \hline 
    
    \rule{0mm}{6mm}
    \multirow{5}{*}{\makecell{\textbf{Fully}\\ \textbf{Adaptive}}} & \citet{cutkosky2024fully} & {$\norm{u}\sqrt{\sumtT \norm{\gt}^2} + \gamma\norm{u}^2/\epsilon + \epsilon G^2/\gamma$}  & $O(d)$\\[1mm] \cline{2-4}

    \rule{0mm}{6mm}
    &    \citet{wang2025parameter} & {$\norm{u}^2\sqrt{V_T^{\smash{+}}} + L^2\norm{u}^4 + G\norm{u}^3 + G^2\norm{u}^2 + G^2\sqrt{\sumtT \norm{\gt}}$}   & $O(d)$\\[1mm] \cline{2-4} 

    \rule{0mm}{6mm}
    & \textbf{Ours},~\pref{thm:simple-optimistic} & {$\norm{u}\sqrt{V_T(u)} + L\norm{u}^2 + \gamma\norm{u}^2/\epsilon + \epsilon G^2/\gamma$}   & $O(d+\log T)$\\[1mm] \cline{2-4}

    \rule{0mm}{6mm}
    & \textbf{Ours},~\pref{thm:efficient-optimistic} & {$\norm{u}\sqrt{V_T(u)} + L\norm{u}^2 + \gamma\norm{u}^2/\epsilon + \epsilon G^4/\gamma^3$}& $O(d)$\\[1mm] \hline

    \hline 
    \end{tabular}}\vspace{-3mm}
\end{table}

\paragraph{Comparator-adaptive regret.} ~We first propose an \emph{optimistic-to-gradient-variation reduction}, which transforms the challenge of achieving comparator or fully-adaptive gradient-variation regret bounds into the problem of attaining standard optimistic regret in online learning. This reduction leverages the negative Bregman divergence terms that naturally arises in regret linearization. 
As a warm-up, we show that
by instantiating this reduction with the existing optimistic comparator-adaptive algorithm from~\citet{jacobsen2022parameter}, we obtain a \emph{comparator-adaptive} gradient-variation bound of $\tilde{O}(\norm{u}\sqrt{V_T(u)} + L\norm{u}^2 + G\norm{u})$ with \emph{efficient} closed-form updates requiring only $O(d)$ computational cost per round, significantly improving the efficiency of the best-known prior result~\citep{wang2025parameter}, which maintains a meta-base ensemble structure 
requiring $O(\log^2 T)$ computation on each round.

\paragraph{Fully-adaptive regret.}
We then design fully-adaptive methods for gradient-variation regret that require no prior knowledge of $\norm{u}$ or $G$, which is the main focus of this work and constitutes the key technical contributions of the paper.
We provide two algorithms.
\emph{(i)} Starting from the comparator-adaptive \mbox{optimistic result} of~\citet{jacobsen2022parameter}, we propose a simple algorithm to incorporate Lipschitz adaptivity, which is accomplished by a \mbox{\emph{virtual clipping technique}} over the optimistic gap in the regularizer and adding a quadratic penalty to form a hybrid regularizer. This yields a fully-adaptive gradient-variation bound $\tilde{O}(\norm{u}\sqrt{V_T(u)} + L\norm{u}^2 + G^2)$, strictly improving the prior best-known result of~\citet{wang2025parameter} whose leading term is $\norm{u}^2\sqrt{V_T^{\smash{+}}}$.
\mbox{However}, its update lacks a closed-form expression and may incur a computational cost of $O(d + \log T)$ per round.
\emph{(ii)} We extend the fully-adaptive non-optimistic algorithm of~\citet{cutkosky2024fully} by equipping it with optimistic guarantees, which is achieved by a \emph{refined optimistic reduction}~\citep{cutkosky2019combining}.
The resulting algorithm admits closed-form updates in $O(d)$ time per round and achieves regret $\tilde{O}(\norm{u}\sqrt{V_T(u)} + L\norm{u}^2 + G^4)$, though suffers a slightly larger lower-order dependence of $G^4$ instead of $G^2$.
A comparison of the results can be found in second half of
\pref{table:results}.

\paragraph{Dynamic regret and the SEA model.}
We extend our fully-adaptive algorithm to optimize \emph{dynamic regret} in~\pref{eq:def-dynamic-regret}. By combining it with a one-dimensional reduction~\citep{COLT'18:black-box-reduction} and an anytime Lipschitz-adaptive algorithm over the unit ball, we obtain a dynamic \mbox{regret} bound with leading terms as $\tilde{O}(\sqrt{(M^2 + MP_T) V_T(u_{1:T})} + L(M^2 + MP_T) + G P_T)$, where $M = \max_t \norm{u_t}$, using $O(d\log t)$ time per iteration $t$.
We further apply this to the \emph{Stochastically Extended Adversarial} (SEA) model~\citep{NeurIPS'22:SEA}, an intermediate setting between adversarial and stochastic convex optimization. Our gradient-variation dynamic regret bound translates to the SEA setting by replacing $V_T(u_{1:T})$ with the sum of stochastic variance and adversarial variation. This yields the \emph{first} dynamic regret guarantee for the SEA model in the unconstrained setting and substantially improves upon~\citet{wang2025parameter}, who only provide static regret and suffer a quadratic dependency $\norm{u}^2\sqrt{V_T^{\smash{+}}}$, where $V_T^{\smash{+}}$ can be much larger than our problem-dependent $V_T(u_{1:T})$.

\paragraph{Notations.}
Let  $[N]=\Set{1,\ldots,N}$ denote the set of integers up to  $N\ge1$. We define $\log_+(\cdot) \define \max\{1, \log(\cdot)\}$. For an indexed collection $\{a_t\}_{t\ge1}$, we also write $(a_t)_t$ when it is clear from the context.
The Bregman divergence with respect to a differentiable convex function $\psi$ is $\cD_{\psi}(x,y) = \psi(x) - \psi(y) - \inner{\grad \psi(y), x - y}$. Unless otherwise specified, $\norm{\cdot}$ denotes the $\ell_2$-norm.
Finally, the $\Otilde(\cdot)$-notation suppresses logarithmic factors, and we treat $\log \log T$ as a constant. 

\paragraph{Organization.} The rest of the paper is organized as follows.~\pref{sec:gradient-variation} provides the results for comparator-adaptive gradient-variation bounds. 
In~\pref{sec:fully-adaptive}, we propose two fully-adaptive algorithms and their analysis.~\pref{sec:applications} presents the applications to dynamic regret and the SEA model.~\pref{sec:conclusion} concludes the paper. All the proofs are deferred to the appendix.

\section{Warm Up: Comparator-Adaptive Gradient-Variation Bounds via Optimism}
\label{sec:gradient-variation}

In this section, we first establish a simple black-box reduction that obtains comparator-dependent gradient variation bounds 
from standard optimistic OLO bounds. We then present as an immediate application a comparator-adaptive gradient variation bound 
using~\citet{jacobsen2022parameter}.

\subsection{Gradient-Variation Black-Box Reduction}

Instead of building gradient-variation bounds from scratch, we provide a simple black-box reduction that leverages the existing results in \emph{optimistic} online learning~\citep{jacobsen2022parameter}.
In addition to receiving standard gradient feedback $\gt\in\partial\ft(\wt)$, an optimistic online algorithm receives an optimistic ``hint'' vector $h_t$ at the beginning of each round, 
which can be leveraged to achieve improved guarantees when the hints accurately predict 
the next gradient, $h_t\approx g_t$.
The following theorem shows that any online algorithm achieving a standard optimistic regret bound (i.e., regret scaling with $\sqrt{\sum_t\|\gt-h_t\|^2}$) can also deliver a gradient-variation regret bound when the losses are $L$-smooth by setting 
the optimistic hint as 
$h_t = g_{t-1}$.

\begin{restatable}{theorem}{SmoothOptimistic}\label{thm:smooth-optimistic}
\!\!\emph{\textbf{(Optimistic-to-Gradient-Variation Reduction)}}
  Let $\cU$ be a class of sequences in $\R^d$, and let $\cA$ be an online learning algorithm that receives $\{g_t\}_{t=1}^T$ as gradients and takes $\{h_t\}_{t=1}^T$ as optimistic hints.
  Suppose $\cA$ guarantees that for any sequence $\cmp_{1:T}\in\cU$ that
  \begin{align}
    \label{eq:reduction-optimism}
    \sumtT \inner{\gt, \wt-\cmp_{t}}\le A_{T}(\cmp_{1:T}) + B_{T}(\cmp_{1:T})\sqrt{\sum_{t=1}^T \norm{\gt-\h_t}^{2}}.
  \end{align}
  Then for any sequence of $G$-Lipschitz, $L$-smooth convex functions $\Set{f_t}_{t=1}^T$,
  by setting $\gt=\grad f_t(\wt)$ and $h_t=g_{t-1}$ , $\cA$ achieves
  \begin{align}
    \label{eq:reduction-GV}
    \reg_{T}(\cmp_{1:T}) \le
      A_{T}(\cmp_{1:T})+ 4LB_{T}(\cmp_{1:T})^2 + 2B_{T}(\cmp_{1:T})\sqrt{ G^2 + V_T(u_{1:T}) + L^2P_T^{\norm{\cdot}^2}(u_{1:T}), } 
  \end{align}
  where
  $P_{T}^{\norm{\cdot}^{2}}(\cmp_{1:T})\define\sum_{t=2}^T\norm{\cmp_{t}-\cmp_{\tmm}}^{2}$ is the squared path length.
\end{restatable}

The proof is provided in~\pref{app:proof-smooth-optimistic}, where the key idea is to appropriately decompose the $\norm{g_t - g_{t-1}}^2$ term and leverage a negative Bregman divergence term that naturally arises from regret linearization~\citep{NeurIPS'24:OptimalGV}.
While the theorem is framed in terms of dynamic regret, it also applies to static regret by considering the class of fixed sequences $\cmp_1=\cdots=\cmp_T$ in $\R^d$.
Similarly, the theorem can also be applied to constrained settings,
so long as smoothness is still defined over the entire space (or at least over 
a slightly augmented space; see, e.g.,~\citet[Appendix~A]{NeurIPS'24:OptimalGV}).

We also remark that the result above simultaneously implies \emph{small-loss} bounds, scaling with 
$F_T(u_{1:T})\define \sumtT (f_t(u_t) -\inf_{w\in\R^d} f_t(w))$. Indeed, in~\pref{app:proof-smooth-optimistic}
we show that 
the obtained bounds more generally scale with $O\big(\sqrt{\Min{L F_T(\cmp_{1:T}),V_T(\cmp_{1:T})}}\big)$.
Thus, each of our results in what follows simultaneously achieve 
small-loss bounds. Throughout the paper we focus our discussion on the gradient variation bounds for simplicity.

\subsection{Implication to Comparator-Adaptive Gradient-Variation Regret}
When the Lipschitz constant $G$ is known \emph{a priori}, we can immediately apply~\pref{thm:smooth-optimistic} to achieve
comparator-adaptive gradient variation bounds using existing algorithms. Indeed,
\citet{jacobsen2022parameter} showed that optimistic Follow-The-Regularized-Leader (FTRL) with a carefully-designed regularizer can achieve comparator-adaptive optimistic regret.\footnote{\citet{jacobsen2022parameter} primarily focus on a mirror-descent-based formulation for dynamic regret minimization. Here, since we focus on static regret, we provide an analysis based on FTRL which significantly streamlines their arguments.} Briefly, it updates by  
\begin{equation}
  \label{eq:FTRL-pf-update}
  w_{t} = \mathop{\arg\min}\limits_{w\in\R^d} \Big\langle h_{t} + \sum_{s=1}^{t-1}g_s, w \Big\rangle + \psi_{t}^{\textsc{pf}}(w,\alpha_t),~~ \psi_{t}^{\textsc{pf}}(w,\alpha) = \int_{0}^{\norm{w}} \min_{\eta\le 1/G}\mbr{ \tfrac{\log(x/\alpha + 1)}{\eta} + \eta\Vbar_t } \mathrm{d}x.
\end{equation}
Here, $\psi_{t}^{\textsc{pf}}(w,\alpha_t)$ is a \emph{parameter-free}\footnote{The regularizers associated with comparator-adaptive guarantees are also sometimes referred to as 
\emph{linearithmic} regularizers, due to their log-linear form \citep{orabona2021parameterfree}. 
} regularizer, 
$\Vbar_t \propto \sum_{s=1}^{t-1}\norm{g_s - h_s}^2$ is the empirical gradient variation, and $(\alpha_t)_t$ is a non-increasing sequence.
Intuitively, this regularizer applies weaker regularization compared to the typical quadratic regularizer $\frac{1}{\eta}\norm{w}^2$, adaptively balancing the trade-off between the $\norm{u}$-term and the empirical gradient variation $\smash{\bar{V}_T}$. This ensures that $\smash{\psi_T^{\textsc{pf}}(u) \le \Otilde( \norm{u} \sqrt{\bar{V}_T} )}$ without any explicit tuning based on $\norm{u}$.
With an analysis similar to that of
\citet{jacobsen2022parameter}, this algorithm guarantees:
\begin{equation}
  \label{eq:FTRL-pf-regret}
  \reg_T(u) \leq \Otilde\Bigg( \norm{u}\sqrt{\sumtT\norm{g_t - h_t}^2} + G\norm{u} + \epsilon G\Bigg).
\end{equation}
A detailed specification of the algorithm can be found in~\pref{alg:optimistic-closed-form} of~\pref{app:pf-oftrl-alg}.
Combined with the optimistic-to-gradient-variation reduction in~\pref{thm:smooth-optimistic}, we obtain the following comparator-adaptive gradient-variation regret, whose proof is provided in~\pref{app:proof-warmup}.

\begin{restatable}{theorem}{Warmup}\label{thm:warmup}
\!\!\emph{\textbf{(Comparator-Adaptive Gradient-Variation Regret)}}
  For any $\cmp\in\R^d$, Algorithm~\ref{alg:optimistic-closed-form} (in Appendix~\ref{app:pf-oftrl-alg}) with $h_t = g_{t-1}$ guarantees
  \begin{align*}
    \reg_{T}(\cmp) \le \tilde{O} \Bigg(\norm{\cmp}\sqrt{V_T(u)\sbr{\log_+{\tfrac{\norm{u}\sqrt{T}}{\epsilon}}}} + L\norm{\cmp}^{2}+ G\norm{u}  + \epsilon G \Bigg).
  \end{align*}
  where we keep the log factors of the dominant term in the $\tilde O(\cdot)$-notation for clarity. Moreover, the algorithm admits an efficient closed-form update with $O(d)$ time per iteration.
\end{restatable}

Importantly, albeit with a seemingly complex form, the update~\pref{eq:FTRL-pf-update} admits an efficient closed-form update formula, as shown in~\pref{alg:optimistic-closed-form}, requiring only $O(d)$ time per round. In contrast, the prior best-known result~\citep{wang2025parameter} requires $O(d\log^2 T)$ computation per-round and is significantly less efficient than ours. This inefficiency arises from their use of an online ensemble that maintains $O(\log^2 T)$ base learners: they derive a range for the comparator norm $\norm{u}$ to prevent the bound from being vacuous, discretize this range, and run a base learner in a bounded domain for each candidate diameter. A meta-algorithm is then used to combine the outputs of base learners.

\section{Fully-adaptive Gradient-Variation Bounds}
\label{sec:fully-adaptive}
Now we target fully-adaptive gradient-variation regret, where neither the comparator norm $\norm{u}$ nor the Lipschitz constant $G$ is known in advance. Leveraging the optimistic-to-gradient-variation reduction, it suffices to design a \emph{fully-adaptive optimistic} algorithm. There are two possible \mbox{approaches:}
\begin{enumerate}
  \item[\emph{(i)}] Begin with the comparator-adaptive optimistic algorithm~\citep{jacobsen2022parameter} and enhance it with Lipschitz adaptivity, as described in \pref{subsec:simple-optimistic}.
  \item[\emph{(ii)}] Alternatively, start from the fully-adaptive non-optimistic algorithm of \citet{cutkosky2024fully} and extend it to its optimistic counterpart, as presented in \pref{subsec:efficient-optimistic}.
\end{enumerate}  
We will discuss the challenges inherent in these extensions when using existing techniques and introduce new ideas to overcome them. The algorithm in \pref{subsec:simple-optimistic} is conceptually simple and enjoys optimal regret bounds, but it lacks a closed-form update and requires a line search to implement.
Alternatively, the approach in \pref{subsec:efficient-optimistic} extends the algorithm of~\citet{cutkosky2024fully} using a new optimistic reduction, leading to an algorithm having an efficient closed-form update while maintaining nearly the same regret bounds, with only slight deterioration in horizon-independent lower-order terms.

\subsection{A Simple Algorithm via Virtual Clipping}
\label{subsec:simple-optimistic}

When the Lipschitz constant $G\ge\max_{t\in[T]}\norm{g_t}$ is unknown, it is natural to employ the observable maximum gradient norm $\smash{\hat G_t=\max_{s<t} \norm{\gs}}$ as a guess of $G$. This motivates the \mbox{\emph{gradient clipping}} technique from~\citet{cutkosky2019artificial}, which is now a standard approach to obtain Lipschitz-adaptivity. 
In particular, 
\citet{cutkosky2019artificial}
proposes feeding the online algorithm with \emph{clipped}
gradients $\tgt\define \gt\min\{1, \hat G_t/\norm{\gt}\}$, 
which are instead bounded by the \emph{known} Lipschitz constant $\hat G_t$.
Then this clipping 
reduces the problem to regret against a 
sequence of {clipped}
gradients in a \emph{black-box} \mbox{manner:}
\begin{equation}
    \sumtT \inner{\gt,\wt-\cmp}
    \le G\left(\norm{\cmp}+\max_{t\in[T]}\norm{\wt}\right) + 
    \sumtT \inner{\tgt,\wt-\cmp} .\label{eq:black-box-clipping}
\end{equation}
However, in unbounded domains the term $\max_t\norm{\wt}$ is typically difficult to control, \mbox{leading to an} additional \emph{cubic} penalty $G\norm{\cmp}^3$ in 
prior works \citep{cutkosky2019artificial,mhammedi2020lipschitz}.

Our key insight is that
instead of applying a black-box clipping argument,
feeding the \emph{true} gradients $\gt$ to the algorithm while
including an additional quadratic regularizer allows one to apply a \emph{virtual} clipping 
argument which incurs only an $O(\norm{\cmp}^2)$ overhead.
Crucially, the $\norm{u}^3$ vs. $\norm{u}^2$ dependency is a \emph{key distinction} as the latter matches the $L\norm{u}^2$ term in the optimal gradient-variation regret~\citep{COLT'12:VT},
and in the context of our gradient-variation bounds from the previous section, a $\norm{u}^2$ overhead is negligible because we already expect to incur an $O(\norm{\cmp}^2)$ term from \pref{thm:smooth-optimistic}.
Moreover, the $\norm{u}^2$ penalty only dominates the worst-case bound $O(\norm{u}^2+\norm{u}\sqrt{T})$ when that bound is vacuous (i.e., $\norm{u}=\Omega(\sqrt{T})$). By contrast, a $\norm{u}^3$ penalty can dominate even in non-vacuous regimes, so it cannot be treated as a lower-order term.

\paragraph{Virtual clipping via quadratic regularization.}
To illustrate the crux of the virtual clipping argument, 
consider a standard 
FTRL update, 
$\wtpp = \argmin_{w\in\R^d}\inner{\sum_{s=1}^t\gs, \w}+\Psi_{\tpp}(w)$ where 
$\Psi_\tpp$ is an arbitrary convex regularizer.
Then using the standard 
regret analysis \citep[Theorem~7.1]{orabona2019modern}, it can be shown that if $(\Psi_t)_t$ is a non-decreasing sequence of regularizers,
we have
\begin{equation}
   \sumtT \inner{g_t, w_t - u} \le \Psi_{T+1}(\cmp) + \sumtT \Big(\inner{\gt, \wt-\wtpp}- \cD_{\Psi_t}(\wtpp,\wt)\Big) =: \Psi_{T+1}(\cmp) + \sumtT \delta_t. \label{eq:regret-for-clipping}
\end{equation}
Inside the stability term $\sumtT\delta_t$, if we replace $\gt$ with their
clipped quantities $\tgt$, then
\begin{equation}
    \sumtT \delta_t
    \le \frac{G^2}{2\beta} + \sumtT \Big(\inner{\tgt, w_t - w_{t+1}} +
    \frac{\beta}{2}\norm{\wt-\wtpp}^2 - \cD_{\Psi_t}(\wtpp,\wt)\Big),\label{eq:virtual-clipping}
\end{equation}
where we have used Fenchel-Young inequality to bound $\norm{\gt-\tilde\gt}\norm{\wt-\wtpp}\le \frac{1}{2\beta}\norm{\gt-\tilde\gt}^2+\frac{\beta}{2}\norm{\wt-\wtpp}^2$ and used an argument similar to \citet{ICML'19:Ashok-acceleration} to bound $\sumtT \norm{\gt-\tilde\gt}^2 = O(G^2)$.
Now the differences $\frac{\beta}{2}\norm{\wt-\wtpp}^2$ can be cancelled completely by
simply including a matching quadratic penalty in $\Psi_t$ to form a hybrid regularizer:
$\Psi_t(w) = \psi_t^{\textsc{pf}}(w) + \frac{\beta}{2}\norm{w}^2$.

The above illustrates why we refer to our approach as \emph{virtual} clipping. 
Unlike \citet{cutkosky2019artificial}, our approach passes the algorithm the true subgradients $\gt$
rather than the clipped $\tgt$, yet by including a quadratic penalty in $\Psi_t$
we are able to replace the $\gt$ appearing in the stability term 
with the clipped quantity \emph{in the analysis}, allowing us to control this term using standard comparator-adaptive regularizers, such as the one discussed in the previous section, defined in terms of the clipped 
gradients $\tilde g_t$.

Notably, the discussion above easily extends 
to optimistic updates by replacing $\gt$ with $\Delta_t\define\gt-h_t$ and $\tgt$
with $\smash{\hat\Delta_t \define \Delta_t \Min{1, \max_{s<t}\|\Delta_s\|/\|\Delta_t\|}}$.
The resulting algorithm is shown in \pref{alg:simple-optimistic},
and its regret guarantee is presented in \pref{thm:simple-optimistic},
with proof in~\pref{app:proof-simple-optimistic}.

\begin{algorithm2e}[t]
  \caption{Simple Fully-adaptive Optimistic Algorithm}
  \label{alg:simple-optimistic}
  \textbf{Input: }$\epsilon>0,\gamma>0$\\
  \textbf{Initialize: }$\w_{1} = \zeros$, $h_{1}=\zeros$, $\hat{M}_1 = \gamma$, $\beta = \frac{\gamma}{\epsilon}$ \\
  \textbf{Define: }Regularizer $\psi^\textsc{pf}(w;\alpha,V,M) = 3\int_{0}^{\norm{w}} \min_{\eta\le 1/M}\mbr{ \frac{\log(x/\alpha+ 1)}{\eta} + \eta V } \mathrm{d}x$ \\
  \For{$t=1:T$}{
    Play $\wt$, receive gradient $g_{t}=\nabla f_{t}(\wt)$ and optimistic hint $h_{\tpp}$ \\
    
    Define: \vspace{-2mm}
    \[
    \begin{aligned}
      \Delta_t&=\gt-h_t,&\quad 
       \hat\Delta_t &= \Delta_t\Min{1, \tfrac{\hat M_t}{\norm{\Delta_t}}}, 
      &\quad  \hat M_\tpp&=\max\big\{\hat M_t, \norm{\Delta_t}\big\}\\
        B_\tpp&=4+{\textstyle\sum_{i=1}^t}\tfrac{\|\hat \Delta_i\|^2}{\hat M_i^2},
    &\quad\Vbar_\tpp &= 4\hat M_\tpp^2+{\textstyle\sum_{i=1}^t}\|\hat \Delta_i\|^2,  
    &\quad  \alpha_\tpp &= \tfrac{\epsilon}{\sqrt{B_\tpp}\log^2(B_\tpp)}
    \end{aligned}
    \vspace{-2mm}
    \]
    
    Choose regularizer $\Psi_{t+1}(w) = \psi^{\textsc{pf}}\big(w;\alpha_\tpp,\Vbar_{t+1},\hat{M}_{t+1}\big) + \frac{\beta}{2}\norm{w}^2$\\
    Update $w_{t+1} = \argmin_{w\in\R^d} \inner{h_{t+1} + \sum_{s=1}^{t}g_s, w} + \Psi_{t+1}(w)$
  }
\end{algorithm2e}

\begin{restatable}{theorem}{SimpleOptimistic}\label{thm:simple-optimistic}
For any $\cmp\in\R^d$, 
Algorithm~\ref{alg:simple-optimistic} with $h_t=g_{t-1}$ guarantees
    \begin{align*}
      \reg_T(u)
      &\le \tilde O\Bigg( \norm{u} \sqrt{V_T(u)\brac{\log_+\tfrac{\norm{u}\sqrt{T}}{\epsilon}}} + \big( L + \tfrac{\gamma}{\epsilon} \big)\norm{u}^2 + \gamma\norm{u} + \epsilon G\big(\tfrac{G}{\gamma} + 1 \big) + \epsilon\gamma \Bigg).
    \end{align*}
\end{restatable}

Ideally, we would set $\epsilon$ in terms of $\norm{u}$ and $\gamma$ in terms of $G$, so that these low-order terms would correspond to the appropriate units of $G\norm{u}$. Nonetheless, observe that we may set these parameters naively (e.g, $\epsilon=\gamma=1$) without 
significantly impacting the main terms in the bound.

\pref{thm:simple-optimistic} strictly improves over the previous best-known bound for fully-adaptive gradient-variation regret 
$\Otilde\big(\norm{u}^2\sqrt{V_T^{\smash{+}}} + L^2\norm{u}^4 + G\norm{u}^3 + G^2\norm{u}^2 + G^2\sqrt{\sumtT \norm{\gt}}\big)$ from \citet{wang2025parameter},
improving both the leading terms and constant penalties.
Moreover,~\pref{alg:simple-optimistic} is relatively \emph{simple} compared to alternative approaches and analyses~\citep{cutkosky2024fully} because it lets us immediately apply existing parameter-free regularization strategies---all we had to do was add an additional quadratic regularizer.

\paragraph{Computational efficiency.}
Although~\pref{alg:simple-optimistic} is conceptually simple, computing the update requires careful implementation to ensure computational efficiency.
In each round the direction of $w_{t+1}$ is straightforward to obtain, but its magnitude must be determined by solving an equation that, in general, admits no closed-form solution and requires a line search.
Fortunately, our algorithm guarantees that $\norm{w_{t+1}}$ remains close to $\norm{w_t}$, so one can efficiently search for $\norm{w_{t+1}}$ within a small interval around $\norm{w_t}$ to achieve a sufficiently small approximation error.
The following proposition states useful constraints for computing $w_{t+1}$, with full details deferred to~\pref{app:efficiency}.
\begin{restatable}{myProp}{Efficiency}\label{prop:efficiency}
  Algorithm~\ref{alg:simple-optimistic} guarantees that, for any $t\in[T]$, the direction of $w_{t+1}$ is \mbox{determined by}
  \vspace{-1.5em}
  \begin{align*}
      \frac{w_{t+1}}{\norm{w_{t+1}}} = - \frac{h_{t+1} + \sum_{s=1}^t g_s}{\norm{h_{t+1} + \sum_{s=1}^t g_s}}.
  \end{align*}
  Moreover, denoting $\alpha\coloneq\alpha_\tpp$, $V\coloneq\Vbar_{t+1}$, and $M\coloneq\hat{M}_{t+1}$ the magnitude of $w_{t+1}$ satisfies
  \begin{align*}
     &\qquad\qquad\qquad\qquad\qquad\quad \big| \norm{\w_{t+1}} - \norm{w_t} \big| \le \frac{\norm{g_t - h_t + h_{t+1}}}{\beta}, \quad \text{and} \\
    &\norm{h_{t+1} + \sum_{s=1}^{t}g_s} =
    \begin{cases}
    6\sqrt{V\ln(\norm{w_{t+1}}/\alpha + 1)} + \beta \norm{w_{t+1}}, & \text{if } \sqrt{\frac{\ln(\norm{w_{t+1}}/\alpha + 1)}{V}} \le \frac{1}{M}, \\
    3\sbr{M\ln(\norm{w_{t+1}}/\alpha + 1) + \frac{V}{M}} + \beta \norm{w_{t+1}}, & \text{otherwise}.
    \end{cases}
  \end{align*}
\end{restatable}
\pref{prop:efficiency} implies that, at each round, the direction of $w_{t+1}$ is determined exactly, and the norm of $w_{t+1}$ lies in an interval of length at most $O(G/\beta)$. This interval can be efficiently searched via binary search to obtain an $\epsilon$-approximate solution.
For example, to achieve up to $\epsilon=O(1/T)$ accuracy, it suffices to use $O(\log T)$ iterations to check the above 1-D equation per round, \mbox{leading} to an overall $O(d + \log T)$ per-round time complexity, which actually improves upon the $O(d\log^2 T)$ complexity of the ensemble method of~\citet{wang2025parameter},
while also avoiding their assumption of \mbox{a \emph{known} Lipschitz constant.}
Hence, our approach improves the state-of-the-art both in terms of
regret and computational overhead.
Given the conceptual simplicity of the algorithm and its analysis, we believe this 
approach will be of independent interest for Lipschitz adaptivity in online learning.

\subsection{An Efficient Algorithm via Optimistic Reduction}
\label{subsec:efficient-optimistic}

In this part, we present an efficient fully-adaptive optimistic algorithm by extending the algorithm of \citet{cutkosky2024fully} using a refined optimistic reduction. 

\paragraph{Challenge in extending~\citet{cutkosky2024fully} to optimistic updates.}
\citet{cutkosky2024fully} proposed an efficient algorithm that is free of both $G$ and $\norm{u}$, ensuring an $\Otilde\big(\norm{u}\sqrt{\sum_{t} \|g_t\|^2} + \norm{u}^2 + G^2 \big)$ regret. 
Unfortunately, their approach requires a very delicate analysis 
involving low-level interactions between the gradient clipping reduction \citep{cutkosky2019artificial}
and an internal application of a constraint-set reduction due to \citet{COLT'18:black-box-reduction}, making it difficult to extend beyond the original scope.

In fact, 
the extension to optimistic updates is particularly problematic, as 
the constraint-set reduction adds 
an additional term to the feedback that can ruin 
the optimistic guarantee, as detailed  
by \citet{cutkosky2019artificial} and subsequently 
observed in several 
follow-up related works
\citep{cutkosky2019artificial,bhaskara2020onlinelearning,NIPS'21:LogarithmicHints,JMLR'25:Efficient-NS}. 
Indeed,  
the constraint-set reduction 
of \citet{COLT'18:black-box-reduction} modifies 
the learner's feedback at time $t$ to $\elltilde_t\define\gt + \norm{\gt}\grad S(\wt)$, where $\grad S(\wt) \in\partial \norm{\wt -\Pi_\ww(\wt)}$ and satisfies $\norm{\grad S(\wt)}\le 1$.\footnote{Refinements of the constraint-set reduction exist which slightly improve constant factors and take a slightly different surrogate penalties, but these
still suffer the same incompatibility with optimism described above~\citep{cutkosky2020parameter}.} In the context of 
an optimistic update, this additional term can dominate the desired optimistic dependency, since
$\sumtT \|\elltilde_t-h_t\|^2= O(\sumtT \norm{\gt-h_t}^2 + \norm{\gt}^2)$. 
This \emph{incompatibility between optimism and the constraint-set reduction} 
in turn makes it highly non-trivial to extend the approach of \citet{cutkosky2024fully} to an optimistic guarantee in any \emph{white-box manner}.

\begin{algorithm2e}[!t]
\caption{Efficient Fully-adaptive Optimistic Algorithm}
  \label{alg:efficient-optimistic}
  \textbf{Input: }$\epsilon>0, \gamma>0$\\
  \textbf{Initialize: } Instantiate two instances of the algorithm of~\citet{cutkosky2024fully}, $\cA_x(\epsilon,\gamma)$ applied on $\R^{d}$, and 
  $\cA_y(\epsilon/\gamma, \gamma^2)$ applied on $\R$.\\
  \For{$t=1:T$}{
  Receive optimistic hint $h_t$\\
    Get $\xt\in\R^{d}$ from $\cA_{x}$ and $\yt\in\R$ from $\cA_{y}$\\
    Play $\wt = \xt - \yt h_t$ and observe $\gt = \grad f_{t}(\wt)$\\
    Pass $\gt$ to $\cA_{x}$ as the $t^{\text{th}}$ gradient\\
    Pass $-\inner{\gt,h_t}$ to $\cA_{y}$ as the $t^{\text{th}}$ gradient
  }
\end{algorithm2e}

\paragraph{Refined optimistic reduction.}To address this issue, we avoid this incompatibility by using a refined version of the optimistic reduction from~\citet{cutkosky2019combining}, which directly converts the regret to an optimistic form in a \emph{black-box manner}. 
In this way, the aforementioned constraints are applied internally within the base algorithms, rather than externally on top of them, allowing us to retain the efficiency benefits
of \citet{cutkosky2024fully} without ruining the 
optimistic guarantees. However, this reduction is 
typically applied under the assumption 
that $\norm{h_t}\le 1$ for all $t$, which is not suitable for 
our Lipschitz adaptive setting. The following
proposition, proven in \pref{app:black-box-optimism-partial}, provides a refinement of the 
reduction which accounts for a time-varying 
\mbox{Lipschitz constant. }
\begin{restatable}{myProp}{BlackBoxOptimismPartial}
\label{prop:black-box-optimism-partial}\!\!\emph{\textbf{(Refined Optimistic Reduction)}}
  Let $\cA_x$ and $\cA_y$ be online algorithms defined on $\ww_x=\R^d$ and $\ww_y=\R$ respectively. Suppose the following conditions hold:
  \begin{itemize}\vspace{-2mm}
  \item For each $z\in\Set{x,y}$, $\cA_z$
  guarantees $\reg_T^{\cA_z}(\cmp)\le A_T^{\cA_z}(\cmp)+B_T^{\cA_z}(\cmp)\sqrt{\sumtT \norm{\gt}^2}$ for any $\cmp\in\ww_z$ and $\smash{\{\gt\}_{t=1}^{T}}$ in $\ww_z$, where $\smash{A_{T}^{\cA_{z}}}$ and $\smash{B_T^{\cA_z}}$ are non-negative functions.\vspace{2mm}
  \item $A_T^{\cA_y}$ is non-decreasing and $B_{T}^{\cA_{y}}(\cmpy)\le \abs{\cmpy}\lambda_{T}(\cmpy)$ for some non-decreasing function \mbox{$\lambda_{T}(\cmpy)\ge 1$.}
  \end{itemize}
  Then
  for any $\cmp\in\R^{d}$, Algorithm~\ref{alg:optimistic-reduction} (in Appendix~\ref{app:black-box-optimism-partial}) enjoys the following optimistic bound:
  \begin{align*}
    \sumtT \inner{\gt, \wt-\cmp}
    &\le
      A_{T}^{\cA_{x}}(\cmp)+A_{T}^{\cA_{y}}(\cmpy) + 2B_{T}^{\cA_{x}}(\cmp)
      \sqrt{H_{T}^{2}\lambda_{T}(\cmpy)^{2}+\half \sbrac{\sumtT \norm{\gt-h_{t}}^{2}-\norm{h_t}^2}_{+}},
  \end{align*}
  where $H_T = \max_{t\in[T]}\norm{h_t}$, $\cmpy=B_{T}^{\cA_{x}}(\cmp)/H_{T}$, and $[x]_+\define \max\{x,0\}$.
\end{restatable}

Applying this refined optimistic reduction with $\cA_x$ and $\cA_y$ being two instances of the algorithm of \citet{cutkosky2024fully} (with regret guarantee restated in~\pref{lemma:cutkosky2024fully}) 
leads to the full algorithm  summarized in~\pref{alg:efficient-optimistic}. We then have the following fully-adaptive regret guarantee with proof in \pref{app:proof-efficient-optimistic}.

\begin{theorem}
  \label{thm:efficient-optimistic}
  For any $u\in\R^d$,
  Algorithm~\ref{alg:efficient-optimistic} with hints $h_t=g_{t-1}$
  guarantees $\reg_T(u)$ bounded by
  \begin{align*}
    \Otilde\left( \norm{u}\sqrt{ V_T(u) \sbr{\log_+ \tfrac{\norm{u}G\sqrt{T}}{\epsilon\gamma}} } + \big(L + \tfrac{\gamma}{\epsilon} + \tfrac{\gamma^3}{\epsilon G^2}\big)\norm{u}^2 + \epsilon G\big( \tfrac{G^3}{\gamma^3} + \tfrac{G}{\gamma} \big) + \epsilon \gamma \right).
\end{align*}
  Moreover, the algorithm admits an efficient closed-form update with $O(d)$ time per iteration.
\end{theorem}

\pref{alg:efficient-optimistic} is the \emph{first} efficient and fully-adaptive optimistic algorithm for unconstrained settings.
Compared to \pref{thm:simple-optimistic}, it maintains the same favorable leading term related to gradient variation.
However, the lower-order term slightly deteriorates from $G^2$ to $G^4$,
because the algorithm $\cA_y$ receives feedback $\inner{\gt,h_t}$, and hence has an effective 
Lipschitz constant of $G^2$ when setting $h_t=g_{t-1}$. This then leads to the $G^4$ penalty when applying the fully-adaptive guarantee in~\citet{cutkosky2024fully}. We suspect that this is an artifact of the analysis, though it is currently unclear how 
to further refine the black-box reduction to avoid this penalty while still allowing the proper 
cancellations in the analysis. We leave this as an important direction \mbox{for future work.}

\section{Applications}
\label{sec:applications}

In this section, we extend our parameter-free gradient variation results to two important applications: dynamic regret minimization and the stochastically-extended adversarial (SEA) setting.

\subsection{Dynamic Regret}

Previous works on gradient-variation dynamic regret minimization have primarily focused on constrained domains~\citep{NeurIPS'20:sword,JMLR'24:Sword++}. In this section we provide comparator-adaptive and fully-adaptive algorithms achieving dynamic gradient variation bounds in \emph{unconstrained} settings.

\paragraph{Comparator-adaptive dynamic regret.}
When the Lipschitz constant is known, we can simply apply an optimistic extension of the
comparator-adaptive dynamic regret algorithm
of~\citet{jacobsen2022parameter} (e.g., adding an optimistic step, as in their Algorithm 3, to the base learner in their Algorithm 2),
followed by \pref{thm:smooth-optimistic} to get  
comparator-adaptive gradient-variation dynamic regret.

\begin{restatable}{theorem}{SimpleDynamic}\label{thm:simple-dynamic}
    For any sequence $\cmp_1,\ldots,\cmp_T$ in $\R^d$, 
    the optimistic extension of \citet[Algorithm 2]{jacobsen2022parameter} with $h_t=\gtmm$ guarantees
    \begin{align*}
    \reg_T(\cmp_{1:T})
    &\le 
    \tilde O\bigg(
    \sqrt{(M^2+MP_T)V_T(\cmp_{1:T})\brac{\log_+\tfrac{MT}{\epsilon}}} + LMP_T + LM^2 +  \epsilon G \bigg),
    \end{align*}
    where $M \define \max_{t\in[T]}\norm{u_t}$.
\end{restatable}

While this result is a simple extension of existing results when applied with our reduction in 
\pref{thm:smooth-optimistic}, 
we note that this is the first instance of an explicit gradient-variation bound for dynamic regret in unconstrained domains, and hence 
fills an important gap in the literature.

\paragraph{Fully-adaptive dynamic regret.}
With an unknown Lipschitz constant,
directly extending our previous fully-adaptive results to dynamic regret is challenging.
For our first approach in~\pref{subsec:simple-optimistic},
including a quadratic regularizer in unconstrained domains significantly complicates 
dynamic regret guarantees~\citep{jacobsen2023unconstrained}, 
and our black-box approach in~\pref{subsec:efficient-optimistic}
also does not naturally extend to dynamic regret, since~\citet{cutkosky2024fully} provides only \emph{static} regret bounds based on FTRL, and extending their analysis to dynamic regret is highly non-trivial. 

Instead, we can utilize a black-box reduction from~\citet{COLT'18:black-box-reduction}, which decomposes the regret into an unconstrained \emph{static regret} problem~---~wherein we can apply our results from the previous section~---~and a
dynamic regret problem \emph{on the unit ball}, where existing 
algorithms such as \textsc{Sword}~\citep{NeurIPS'20:sword} and \textsc{Sword++}~\citep{JMLR'24:Sword++} can be deployed.
Specifically, the decision $w_t$ is decomposed as $w_t = y_t x_t$, with one online algorithm producing the magnitude $y_t \in \R$ and another producing the direction $x_t$.
Applied to dynamic regret, this reduction guarantees the following decomposition~\citep[Appendix~J]{jacobsen2022parameter}:
\begin{equation*}
    \reg_T(u_{1:T}) \le
    \reg_T^{\onedim}(M) + M\reg_T^{\cB}(\cmp_{1:T}/M),
\end{equation*}
where $\reg_T^{\onedim}(M)$ is the 1D static regret with comparator $M \define \max_{t \in [T]} \norm{u_t}$, and $\reg_T^{\cB}(\cmp_{1:T}/M)$ is the dynamic regret of the direction learner on the unit ball w.r.t the re-scaled comparator sequence.

To obtain gradient-variation guarantees, we further incorporate the optimistic scheme into this reduction and summarize our algorithm in~\pref{alg:dynamic}.
For the 1D algorithm $\cA_{\onedim}$, we employ our fully-adaptive optimistic \pref{alg:efficient-optimistic}.
For the bounded domain algorithm $\cA_{\cB}$, we use a standard online ensemble method similar to~\citet{JMLR'24:Sword++}, with an enhancement that applies time-varying step sizes for base learners to achieve Lipschitz adaptivity, and
the doubling trick to guide the instantiation of new base learners.
This online ensemble method is summarized in
\pref{alg:anytime-dynamic} in \pref{app:anytime-dynamic}, with the Lipschitz-adaptive and anytime theoretical guarantee in \pref{thm:anytime-dynamic}.

Finally, we provide our fully-adaptive and anytime dynamic regret guarantee in~\pref{thm:dynamic}, with a more detailed version and the proof in~\pref{app:proof-dynamic}.

\begin{algorithm2e}[t!]
\caption{Fully-adaptive Gradient-Variation Dynamic Regret Minimization}
  \label{alg:dynamic}
  \textbf{Input: }$\epsilon>0,\gamma>0$\\
  \textbf{Initialize: } $g_0 = \zeros$. Instantiate~\pref{alg:efficient-optimistic} as $\cA_{\onedim}(\epsilon,\gamma)$ acting on $\R$, and instantiate~\pref{alg:anytime-dynamic} as $\cA_{\cB}(D:=1,\gamma)$ acting on the unit ball $\cB = \{w\in\R^d:\norm{w}\le 1\}$ \\
  \For{$t=1:T$}{
    Set optimistic hint $h_t = g_{t-1}$\\
    Pass $h_t$ to $\cA_{\cB}$ as the $t^{\text{th}}$ hint, and get $\xt$ from $\cA_{\cB}$\\
    Pass $\inner{h_t, x_t}$ to $\cA_{\onedim}$ as the $t^{\text{th}}$ hint, and get $y_t$ from $\cA_{\onedim}$\\
    Play $\wt = y_t x_t$ and observe $\gt = \grad f_{t}(\wt)$\\
    Pass $\gt$ to $\cA_{\cB}$ as the $t^{\text{th}}$ gradient, and
    pass $\inner{\gt, x_t}$ to $\cA_{\onedim}$ as the $t^{\text{th}}$ gradient
  }
\end{algorithm2e}

\begin{theorem}
    \label{thm:dynamic}
    For any sequence $u_1,\ldots,u_T\in\R^d$, Algorithm~\ref{alg:dynamic}  guarantees
    \mbox{$\reg_T(u_{1:T})$ bounded by}
    \begin{align*}
    &\Otilde \Bigg(
        \sqrt{\brac{M^2 + MP_T} \Min{V_T(u_{1:T}), L F_T(u_{1:T})}\brac{\log_+\tfrac{MG\sqrt{T}}{\epsilon\gamma}} } \\[-2mm]
    &\qquad + (LM + G + \gamma) P_T + \big(L + \tfrac{\gamma}{\epsilon} + \tfrac{\gamma^3}{\epsilon G^2}\big)M^2 + \gamma M + \epsilon G\big( \tfrac{G^3}{\gamma^3} + \tfrac{G}{\gamma} + 1 \big) + \epsilon \gamma \Bigg),
\end{align*}
where $M \define \max_{t\in[T]}\norm{u_t}$ and $F_T(u_{1:T})\define \sumtT (f_t(u_t) - \inf_{w\in\R^d} f_t(w))$. Moreover, the algorithm runs in $O(d\log t)$ time on iteration $t$.
\end{theorem}
\pref{thm:dynamic} provides the first fully-adaptive gradient-variation dynamic regret in the unconstrained setting.
It matches the best-known bound in the constrained setting~\citep{JMLR'24:Sword++} up to logarithmic factors.
By the optimistic-to-gradient-variation reduction, this theorem simultaneously achieves a small-loss bound that can also be generally applied to static regret.
Due to gradient-variation adaptivity, this result also readily applies to the stochastically extended adversarial model~\citep{NeurIPS'22:SEA}, 
yielding the \emph{first} unconstrained dynamic regret in this setting, formally discussed below.

\subsection{Stochastically-extended Adversarial Model}
\label{subsec:SEA}

The stochastically-extended adversarial (SEA) model~\citep{NeurIPS'22:SEA} is an intermediate setting between adversarial OCO and Stochastic Convex Optimization (SCO).
The key difference from standard OCO is that the losses $f_t$ are sampled from a distribution $\fD_t$ chosen by the environment on each round.
This setting seamlessly interpolates between (potentially non-stationary) SCO 
when $\fD_t$ is chosen obliviously to the learners decisions---naturally modelling 
the non-stationary found in many real-world applications 
such as the Online Label Shift problem~\citep{NeurIPS'22:label-shift, ICDM'23:NOLS,NeurIPS'23:OLS-baby}---and the adversarial setting when $\fD_t=\delta_{\bar{f_t}}$ for \mbox{arbitrary $\bar{f_t}$}.

For the SEA model, the natural performance measure is the expected dynamic regret, i.e.,
\begin{align*}
    \E[\reg_T(u_{1:T})] \define \E\Bigg[\sumtT f_t(w_t) - \sumtT f_t(u_t)\Bigg],
\end{align*}
against a sequence of comparators $u_{1:T}$ that capture the non-stationarity of the environment, where the expectation is taken over the randomness of the sampled functions $f_t$.
We assume that the comparator sequence is oblivious, as elaborated later in~\pref{remark:SEA-oblivious}.

Letting $F_t(w)\coloneq\E_{f_t\sim \fD_t}[f_t(w)]$, \citet{NeurIPS'22:SEA} introduced two quantities characterizing the difficulty of the environment: the stochastic variance $\sigma_t^2 \define \sup_{w \in \W} \mathbb{E}_{f_t \sim \fD_t} [\|\nabla f_t(w)-\nabla F_t(w)\|^2]$,  and the adversarial variation $\Sigma_t^2 \define \sup_{w \in \W}\|\nabla F_t(w)-\nabla F_{t-1}(w)\|^2$.
We consider the following comparator-adaptive generalizations based on an arbitrary $u\in\R^d$:
\begin{equation}
    \label{eq:def-SEA-sigmas}
    \sigma_t^2(u) \define \mathbb{E}_{f_t \sim \mathfrak{D}_t}\left[\left\|\nabla f_t(u)-\nabla F_t(u)\right\|^2\Big| \cF_\tmm\right], \text { and } \Sigma_t^2(u)\define \left\|\nabla F_t(u)-\nabla F_{t-1}(u)\right\|^2,
\end{equation}
where we denote $\cF_\tmm$ the sigma field generated up to the beginning of round $t$.
The following theorem then provides fully-adaptive dynamic regret guarantees for the SEA model, with proof provided in~\pref{app:proof-SEA}.
\begin{restatable}{theorem}{SEA}\label{thm:SEA}
    In the SEA model, for any oblivious comparator sequence $\cmp_1,\ldots,\cmp_T\in\R^d$ such that
    $\max_t\norm{\cmp_t}\le M$ with $M>0$, 
    Algorithm~\ref{alg:dynamic} guarantees
    $\E\mbr{\reg_T(u_{1:T})}$ bounded by
    \begin{align*}
    & \Otilde \Bigg(
        \sqrt{\brac{M^2 + MP_T} \brac{\sigma_{1:t}^2(u_{1:T}) +\Sigma_{1:T}^2(u_{1:T}) }\brac{\log_+\tfrac{MG\sqrt{T}}{\epsilon\gamma}} } \\[-2mm]
    &\qquad + (LM + G + \gamma) P_T + \big(L + \tfrac{\gamma}{\epsilon} + \tfrac{\gamma^3}{\epsilon G^2}\big)M^2 + \gamma M + \epsilon G\big( \tfrac{G^3}{\gamma^3} + \tfrac{G}{\gamma} + 1 \big) + \epsilon \gamma \Bigg).
    \end{align*}
    where $\sigma_{1:T}^2(u_{1:T})\define \sum_{t=2}^T(\sigma_t^2(u_{t-1}) + \sigma_{t-1}^2(u_{t-1}) )$, and $\Sigma_{1:T}^2(u_{1:T})\define \sum_{t=2}^T \Sigma_t^2(u_{t-1})$.
    Moreover, the algorithm runs in $O(d\log t)$ time on iteration $t$. 
\end{restatable}

\begin{myRemark}
    \label{remark:SEA-oblivious}
We assume that the comparator sequence $\cmp_{1:T}$ is oblivious, ensuring $\E[P_T \cdot V_T(\cmp_{1:T})] = P_T \cdot \E[V_T(\cmp_{1:T})]$. This assumption is reasonable in many real-world applications such as the online label shift problem~\citep{NeurIPS'22:label-shift,ICDM'23:NOLS,JMLR'24:OMD4SEA}---a classification task in non-stationary environments where the label distribution changes over time (e.g., species monitoring, where the learner does not affect the underlying environment dynamics). In such settings, the learner competes against optimal parameters $u_t = \argmin_{w}F_t(w)$, which (we can assume) is determined by the underlying environment and is independent of the learner's predictions.
\end{myRemark}

The closely-related work of~\citet[Theorem 4.5]{wang2025parameter} obtained a fully-adaptive static regret of ${\Otilde\big(\norm{u}^2\sqrt{\sigma_{1:T}^2(w_{1:T}) + \Sigma_{1:T}^2(w_{1:T})} + L^2\norm{u}^4 + G\norm{u}^3 + G^2\norm{u}^2 + G^2\sqrt{\sumtT \norm{\gt}}\big)}$ for the SEA model,
suffering from a quadratic dependency on the comparator norm in the leading term and a large penalty of $\norm{u}^4$.
In contrast, our~\pref{thm:SEA} implies a static regret bound $\tilde{O}(\norm{u}\sqrt{\sigma_{1:T}^{{2}}(u) + \Sigma_{1:T}^{{2}}(u) } + L\norm{u}^2 + G^4)$, which has significantly better $\norm{u}$-dependencies.

It is worth noting that while our~\pref{thm:SEA} introduces a lower-order $G^4$ penalty that does not appear in~\citet{wang2025parameter},
this penalty is not present in our algorithm from~\pref{subsec:simple-optimistic}, which when applied 
in the SEA setting obtains a strict improvement over their result while being slightly more efficient.
Moreover, our result is \emph{problem-dependent}, with $\sigma_{1:T}^2(u) + \Sigma_{1:T}^2(u)$,
whereas~\citet{wang2025parameter} obtained \emph{algorithm-dependent} $\sigma_{1:T}^2(w_{1:T}) + \Sigma_{1:T}^2(w_{1:T})$, with the algorithm's trajectory $w_{1:T}$. For a problem-dependent guarantee, one has to upper bound the algorithm-dependent quantity by the worst-case $\sumtT(\sigma_t^2 + \Sigma_t^2)$, which can easily become $O(T)$ in an unbounded domain.

\section{Conclusion}
\label{sec:conclusion}

In this work, we investigate parameter-free gradient-variation online learning. We provided a simple black-box reduction that converts general optimistic regret bounds into gradient-variation bounds, and leverage this reduction to develop 
two fully-adaptive optimistic algorithms achieving 
gradient-variation regret bounds in unconstrained settings.
As a direct application of our approach, we obtain novel gradient-variation guarantees for dynamic regret and for the SEA model, where our results significantly improve the state-of-the-art.

We anticipate several promising directions for future work.
We expect that our fully-adaptive gradient-variation guarantees will find natural applications in areas such as minimax games and accelerated (universal) optimization.
We also believe developing a more refined optimistic reduction to achieve an efficient fully-adaptive optimistic regret bound without the $G^4$ penalty artifact 
would also be valuable in general. 
Finally, it would be interesting to explore more sophisticated solvers to compute the simple fully-adaptive algorithm from \pref{subsec:simple-optimistic}.

\newpage

\acks{
Peng Zhao and Yuheng Zhao were supported by NSFC (62361146852, 62576164), the Fundamental and Interdisciplinary Disciplines Breakthrough Plan of the Ministry of Education of China (No. JYB2025XDXM118). 
Nicol\`{o} Cesa-Bianchi and Andrew Jacobsen acknowledge the financial support from the EU Horizon CL4-2022-HUMAN-02 research and innovation action under grant agreement 101120237, project ELIAS (European Lighthouse of AI for Sustainability). 
This research was also supported by the ``111 Center'' (No. B26023).
}

\bibliography{refs}

\crefalias{section}{appendix}

\newpage
\appendix
\section{Proofs for Section~\ref{sec:gradient-variation}}
\label{app:gradient-variation}

\subsection{Proof of Theorem~\ref{thm:smooth-optimistic}}
\label{app:proof-smooth-optimistic}

In this section we provide proof of the optimistic-to-gradient-variation reduction. The full version of the theorem is re-stated below, which additionally includes the \emph{small-loss} bounds.
\begin{manualtheorem}{\ref{thm:smooth-optimistic}}[Full Version]
  Let $\cU$ be a class of sequences in $\R^d$, and let $\cA$ be an online learning algorithm that receives $\{g_t\}_{t=1}^T$ as gradients and takes $\{h_t\}_{t=1}^T$ as optimistic hints.
  Suppose $\cA$ guarantees that for any sequence $\cmp_{1:T}\in\cU$ that
  \begin{align*}
    \sumtT \inner{\gt, \wt-\cmp_{t}}\le A_{T}(\cmp_{1:T}) + B_{T}(\cmp_{1:T})\sqrt{\sum_{t=1}^T \norm{\gt-\h_t}^{2}}.
  \end{align*}
  Then for any sequence of $G$-Lipschitz, $L$-smooth convex functions $\Set{f_t}_{t=1}^T$,
  by setting $\gt=\grad f_t(\wt)$ and $h_t=g_{t-1}$ , $\cA$ achieves
  \begin{align*}
    \reg_{T}(\cmp_{1:T}) &\le
    A_{T}(\cmp_{1:T}) + 4LB_{T}(\cmp_{1:T})^2\\
      &\qquad+ 2B_{T}(\cmp_{1:T})\sqrt{ \Min{G^2 + V_T(u_{1:T}) + L^2P_T^{\norm{\cdot}^2}(u_{1:T}), 4 LF_T(u_{1:T}) } },
  \end{align*}
  where
  $P_{T}^{\norm{\cdot}^{2}}(\cmp_{1:T})\define\sum_{t=2}^T\norm{\cmp_{t}-\cmp_{\tmm}}^{2}$ is the squared path length,
  and $F_T(\cmp_{1:T})\define\sumtT (\ft(\cmp_t)-\inf_{w\in\R^d}\ft(w))$ is the small loss.
\end{manualtheorem}

\begin{proof}
  By the definition of
  Bregman divergence $\cD_{f}(x,y)=f(x)-f(y)-\inner{\grad f(y),x-y}$,
  we have
  \begin{align*}
    \reg_{T}(\cmp_{1:T})
    &=
      \sumtT f_{t}(\wt)-f_{t}(\cmp_{t})
      =
      \sumtT \inner{\grad f_{t}(\wt),\wt-\cmp_{t}}-\cD_{f_{t}}(\cmp_{t},\wt)\\
    &=
      \reg_{T}^{\cA}(\cmp_{1:T}) -\sumtT \cD_{f_{t}}(\cmp_{t},\wt)\\
    &\le
      A_{T}(\cmp_{1:T})+B_{T}(\cmp_{1:T})\sqrt{\sum_{t=1}^T \norm{\grad f_{t}(\wt)-\grad f_{\tmm}(\wtmm)}^{2}} - \sumtT \cD_{f_{t}}(\cmp_{t},\wt).
  \end{align*}
  Now apply~\pref{lemma:tedius-calculations}
  to get
  \begin{align*}
    \sumtT \norm{\grad f_{t}(\wt)-\grad f_{\tmm}(\wtmm)}^{2}
    &\le
    \Min{\norm{\nabla f_1(w_1)}^2 + 4V_T(\cmp_{1:T}) + 4L^2P_T^{\norm{\cdot}^2}(\cmp_{1:T}), 16LF_T(u_{1:T}) } \\
    &\qquad
    + 16 L\sumtT \cD_{\ft}(\cmp_t,\wt).
  \end{align*}
  where we define
  \begin{align*}
  V_{T}(\cmp_{1:T})&\define\sum_{t=2}^T \norm{\grad f_{t}(\cmp_{t-1})-\grad f_{\tmm}(\cmp_{t-1})}^{2},\\
  F_T(\cmp_{1:T})&\define\sumtT (\ft(\cmp_t)-\inf_{w\in\R^d}\ft(w)),\qquad
  P_{T}^{\norm{\cdot}^{2}}(\cmp_{1:T})\define\sum_{t=2}^T\norm{\cmp_{t}-\cmp_{\tmm}}^{2}.
  \end{align*}
  Hence, using $\sqrt{a+b}\le \sqrt{a}+\sqrt{b}$ yields
  \begin{align*}
    \reg_{T}(\cmp_{1:T})
    &\le
      A_{T}(\cmp_{1:T}) + 2B_{T}(\cmp_{1:T})\sqrt{ \Min{G^2 + V_T(u_{1:T}) + L^2P_T^{\norm{\cdot}^2}(u_{1:T}), 4 LF_T(u_{1:T}) } }\\
    &\qquad
      +4B_{T}(\cmp_{1:T})\sqrt{L\sumtT \cD_{f_{t}}(\cmp_{t},\wt)}-\sumtT \cD_{f_{t}}(\cmp_{t},\wt)\\
    &\le
      A_{T}(\cmp_{1:T}) + 4LB_{T}(\cmp_{1:T})^2\\
      &\qquad+ 2B_{T}(\cmp_{1:T})\sqrt{ \Min{G^2 + V_T(u_{1:T}) + L^2P_T^{\norm{\cdot}^2}(u_{1:T}), 4 LF_T(u_{1:T}) } },
  \end{align*}
  where the last line uses $ax-bx^{2}\le a^{2}/4b$.
\end{proof}

\subsection{Parameter-free Regularizer from~\citet{jacobsen2022parameter}}

In this part, we introduce a key component of our parameter-free algorithms, which is what we call a ``parameter-free regularizer'' from~\citet{jacobsen2022parameter}.
The following lemma provides the theoretical guarantees.

\begin{restatable}{lemma}{ParameterFreeRegularizer}\label{lemma:pf-regularizer}
Let $g_1,\ldots,g_T$ be an arbitrary sequence of vectors. Suppose $0<M_1\le\ldots\le M_T$ is non-decreasing magnitude hint sequence that $\norm{g_t}\le M_t$ for all $t$, and let $\alpha_1\ge \ldots \ge \alpha_T$ be a non-increasing sequence. Set $\Vbar_t=4M_t^2 + \sum_{i=1}^{t-1}\norm{g_i}^2$ and define:
\begin{align*}
  \psi_t(w) = 3\int_{0}^{\norm{w}} \min_{\eta\le 1/M_t}\mbr{ \frac{\log(x/\alpha_t + 1)}{\eta} + \eta\Vbar_t } \mathrm{d}x,
\end{align*}
then for any sequence $w_1,\ldots,w_{T+1}\in\R^d$:
\begin{align*}
  \sumtT \inner{g_t, w_t - w_{t+1}} - \cD_{\psi_t}(w_{t+1}, w_t) - (\psi_{t+1} - \psi_t)(w_{t+1}) \le \sumtT \frac{2\alpha_t\norm{\g_t}^2}{\sqrt{\Vbar_t}}.
\end{align*}
Moreover, for any $u\in\R^d$:
\begin{align*}
  \psi_{T+1}(u) \le 6\norm{u}\max\bbr{ \sqrt{\bar{V}_{T+1}\log(\norm{u}/\alpha_{T+1} + 1)}, {M}_{T+1}\log(\norm{u}/\alpha_{T+1} + 1) }.
\end{align*}
\end{restatable}

\begin{proof}
This lemma follows from the proof of~\citet[Theorem 6]{jacobsen2022parameter}.
\end{proof}

\subsection{Optimistic FTRL with Parameter-free Regularizer}
\label{app:pf-oftrl-alg}

In this section, we provide a comparator-adaptive optimistic algorithm,~\pref{alg:optimistic-closed-form}, that combines optimistic FTRL with the parameter-free regularizer in~\citet{jacobsen2022parameter}.
This is a slight simplification of~\citet[Algorithm 3]{jacobsen2022parameter} which used a somewhat complicated formulation based on mirror descent in order to develop dynamic regret guarantees; this 
is more general than we need since we focus primarily on static regret. Instead, we provide an FTRL-based formulation which significantly simplifies 
the exposition and analysis, and is likely much easier to follow for most readers familiar with online learning than the centered mirror descent 
argument of \citet{jacobsen2022parameter}.

The algorithm is shown in~\pref{alg:optimistic-closed-form}, which also clearly 
demonstrates that the algorithm can be implemented in $O(d)$ per-round computation. 
The closed-form shown in the pseudocode follows 
from the same arguments as \citet[Theorem 1]{jacobsen2022parameter},
which shows how to compute the equivalent non-optimistic update from the first-order optimality condition,
so we omit the details here for brevity.

\begin{algorithm2e}[!t]
  \SetAlgoLined
  \KwInput{Lipschitz bound $G$, value $\epsilon>0$}\\
  \textbf{Define: }parameter-free regularizer $\psi(w;\bar{V},\alpha) = 3\int_{0}^{\norm{w}} \min_{\eta\le 1/(2G)}\mbr{ \frac{\log(x/\alpha + 1)}{\eta} + \eta\Vbar } \mathrm{d}x$ \\
  \KwInitialize{$V_{1}=4(2G)^{2}$, $\w_{1}=\zeros$, $\theta_{1}=\zeros$, $B_1=4$, $h_1=\zeros$}\\
  \For{$t=1:T$}{
    Play $\wt$, receive subgradient $g_{t}$ and hint $h_\tpp$\\
    Set $\Vbar_\tpp=\Vbar_t + \norm{\gt-h_t}^2$,
    \ \ $B_{t+1} = B_t + \tfrac{\norm{\gt-h_t}^2}{(2G)^2}$,\ \ and $\smash{\alpha_{t+1} = \frac{\epsilon}{\sqrt{B_{t+1}}\log^2(B_{t+1})}}$ \\
    
    Set $\theta_\tpp=\theta_t-\gt$ and $\tilde\theta_\tpp=\theta_\tpp-h_\tpp$\\
    Update
    \begin{align}
        \wtpp &= \argmin_{w\in\R^d} \inner{h_{t+1} + \sum_{s=1}^{t}g_s, w} + \psi(w; \Vbar_\tpp,\alpha_\tpp)\label{eq:pf-optimistic-ftrl} \vspace{3mm}\\
        &=
        \frac{\tilde\theta_\tpp}{\|\tilde\theta_\tpp\|}\alpha_\tpp
      \begin{cases}
      \Exp{\frac{\norm{\tilde\theta_\tpp}^2}{36 \Vbar_\tpp}}-1&\text{if }
      \|\tilde\theta_\tpp\|\le \frac{6\Vbar_\tpp}{(2G)}  \vspace{2mm}\\
      \Exp{\frac{\norm{\tilde\theta_\tpp}}{3(2G)}-\frac{6\Vbar_\tpp}{(2G)}}-1&\text{otherwise}
      \end{cases}\label{eq:optimistic-ftrl-closed-form}
    \end{align}
  }
  \caption{Optimistic FTRL/Centered Mirror Descent}
  \label{alg:optimistic-closed-form}
\end{algorithm2e}

\begin{theorem}[Comparator-adaptive Optimistic Algorithm]
\label{thm:comparator-adaptive-oftrl}
Assume that the hints $h_t\in\R^d$ satisfy $\norm{h_t} \le G$ for all $t$.
For any $\cmp\in\R^d$, Algorithm~\ref{alg:optimistic-closed-form} guarantees
\begin{align*}
    \sumtT\inner{g_t, w_t - u} \le O\brac{
    \norm{u}\sqrt{\sumtT\norm{g_t - h_t}^2\brac{\log_+\tfrac{\norm{u}\sqrt{T}}{\epsilon}}} + \epsilon G + G\norm{u}\brac{\log_+ \tfrac{\norm{u}\sqrt{T}}{\epsilon}}.
    }.
\end{align*}
Moreover, the algorithm admits an efficient closed-form update with $O(d)$ time per iteration.
\end{theorem}

\begin{proof}
For all $t$, let $\psi_t^{\textsc{pf}}(w)\define \psi(w;\Vbar_t,\alpha_t)$ as defined in~\pref{alg:optimistic-closed-form}.
Applying a typical regret bound for optimistic FTRL (e.g., see~\cite[Theorem 7.36]{orabona2019modern}), we have
\begin{align*}
    \sumtT \inner{g_t, w_t - u} \le \psi^{\textsc{pf}}_{T+1}(u) + \sumtT \underbrace{\Big(\inner{g_t - h_t, w_t - w_{t+1}} - \cD_{\psi^{\textsc{pf}}_t}(w_{t+1}, w_t) - (\psi^{\textsc{pf}}_{t+1} - \psi^{\textsc{pf}}_t)(w_{t+1})\Big)}_{\delta_t}.
\end{align*}
Then apply~\pref{lemma:pf-regularizer} with $M_t = 2G \ge \norm{g_t - h_t}$, $\Vbar_t = 4M_t^2 + \sum_{s=1}^{t-1}\norm{g_s - h_s}^2$, and with any non-increasing sequence $\alpha_1\ge \ldots \ge \alpha_T>0$,
\begin{align*}
    \sumtT \delta_t \le \sumtT \frac{2\alpha_t\norm{g_t - h_t}^2}{\sqrt{\Vbar_T}},
\end{align*}
apply~\pref{lemma:tuning-sum} that defines $\alpha_t = \frac{\epsilon}{\sqrt{B_t}\log^2(B_t)}$ where $B_t=4+\sum_{s=1}^{t-1}\norm{g_s - h_s}^2/(2G)^2$, then
\begin{align*}
    \sumtT \delta_t \le 16 \epsilon G.
\end{align*}
And~\pref{lemma:pf-regularizer} also gives us
\begin{align*}
    \psi^{\textsc{pf}}_{T+1}(u) &\le O\brac{\norm{u}\max\bbr{ \sqrt{\bar{V}_{T+1}\log(\norm{u}/\alpha_{T+1} + 1)}, G\log(\norm{u}/\alpha_{T+1} + 1) }} \\
    &\le O\brac{ \norm{u}\max\bbr{ \sqrt{\bar{V}_{T+1}\log(\norm{u}\sqrt{T}/\epsilon + 1)}, G\log(\norm{u}\sqrt{T}/\epsilon + 1) } }.
\end{align*}
Plugging the previous two displays in above yields the stated regret guarantee.
The per-iteration efficiency can be easily seen from the updates in~\pref{eq:optimistic-ftrl-closed-form}.
\end{proof}

\subsection{Proof of Theorem~\ref{thm:warmup}}
\label{app:proof-warmup}

In this section we provide the guarantee for our comparator-adaptive 
gradient variation bound. The result is re-stated below for convenience.
\Warmup*
\begin{proof}
Applying~\pref{thm:comparator-adaptive-oftrl} that guarantees
\begin{align*}
    \sumtT \inner{g_t, w_t - u} \le O\brac{
    \norm{u} \sqrt{\bar{V}_{T}\brac{\log_+\frac{\norm{u}\sqrt{T}}{\epsilon}}}+ \epsilon G + G\norm{u} \brac{\log_+\frac{\norm{u}\sqrt{T}}{\epsilon}} 
    },
\end{align*}
where $\Vbar_{T} = \sumtT\norm{g_t - h_t}^2$.
Hence,~\pref{alg:optimistic-closed-form}
satisfies the conditions of
\pref{thm:smooth-optimistic} with
\begin{equation*}
A_T(u) = O\sbr{ \epsilon G + G\norm{u}\brac{\log_+\frac{\norm{u}\sqrt{T}}{\epsilon}} }, \quad B_T(u) = O\sbr{ \norm{u} \sqrt{\brac{\log_+\frac{\norm{u}\sqrt{T}}{\epsilon}}} }.
\end{equation*}
Plugging these into~\pref{thm:smooth-optimistic} yields the desired result.
\end{proof}
\section{Proofs for Section~\ref{sec:fully-adaptive}}
\label{app:fully-adaptive}

\subsection{Proof of Theorem~\ref{thm:simple-optimistic}}
\label{app:proof-simple-optimistic}

Here we re-state the full version for~\pref{thm:simple-optimistic} of the main text, which additionally includes the derived optimistic regret bound.

\begin{manualtheorem}{\ref{thm:simple-optimistic}}[Full Version]
Assume that the optimistic hints $h_t$ satisfy $\norm{h_t}\le G$ for all $t$.
For any $\cmp\in\R^d$, Algorithm
\ref{alg:simple-optimistic} guarantees
\begin{align*}
    \reg_T(\cmp)
    &\le \tilde O\Bigg(\norm{u} \sqrt{\sumtT \norm{\gt-h_t}^2\brac{\log_+\tfrac{\norm{u}\sqrt{T}}{\epsilon}}} + \epsilon\gamma + \epsilon G + \gamma\norm{u}
     + \tfrac{\gamma}{\epsilon}\norm{u}^2 + \tfrac{\epsilon}{\gamma}G^2 \Bigg).
\end{align*}
Moreover, by setting hints $h_t=g_{t-1}$, for any $u\in\R^d$ it holds that
    \begin{align*}
      \reg_T(u)
      &\le \tilde O\Bigg( \norm{u} \sqrt{V_T(u)\brac{\log_+\tfrac{\norm{u}\sqrt{T}}{\epsilon}}} + L\norm{u}^2 + \epsilon\gamma + \epsilon G + \gamma\norm{u} + \tfrac{\gamma}{\epsilon}\norm{u}^2 + \tfrac{\epsilon}{\gamma}G^2 \Bigg).
    \end{align*}
\end{manualtheorem}
\begin{proof}
Consider the standard optimistic FTRL bound (setting $h_1=h_{T+1}=\zeros$):
\begin{align*}
    \sumtT\inner{g_t, w_t - u} 
    &\le
    \Psi_{T+1}(u) + \sumtT \underbrace{\inner{\gt-h_t, \wt-\wtpp}- \cD_{\Psi_t}(\wtpp,\wt) - \brac{\Psi_\tpp-\Psi_t}(w_\tpp)}_{\delta_t}.
\end{align*}
We proceed by bounding the terms $\sumtT \delta_t$. We begin by replacing $\Delta_t = g_t-h_t$ with a \emph{clipped} version $\hat\Delta_t = \Delta_t\cdot\min\bbr{1, \frac{\hat M_t}{\norm{\Delta_t}}}$, where $\hat M_t = \max_{s<t}\norm{\Delta_t}$. To do this, let $\Psi_t(w) = \psi_t(w) + \frac{\beta}{2}\norm{w}^2$, where $\psi_t(w) \define \psi^{\textsc{pf}}(w;\alpha_t,\Vbar_t,\hat{M}_t)$ and $\beta>0$, then we have:
\begin{align*}
\sumtT \delta_t
&=\sumtT \inner{\Delta_t,\wt-\wtpp}-\cD_{\Psi_t}(\wtpp,\wt) - (\Psi_\tpp-\Psi_t)(\wtpp)\\
&=\sumtT \inner{\hat \Delta_t, \wt-\wtpp}- \cD_{\psi_t}(\wtpp,\wt) - (\psi_\tpp - \psi_t)(\wtpp)\\
&\qquad 
+\sumtT\inner{\Delta_t-\hat\Delta_t,\wt-\wtpp}- \frac{\beta}{2}\norm{w_t - w_{t+1}}^2 \\
&\le\sumtT\inner{\hat \Delta_t, \wt-\wtpp}- \cD_{\psi_t}(\wtpp,\wt) - (\psi_\tpp - \psi_t)(\wtpp) + \frac{1}{2\beta}\sumtT \norm{\Delta_t - \hat{\Delta}_t}^2.
\end{align*}
By definition of the clipped optimistic gap $\hat\Delta_t$, we have $\|\Delta_t - \hat{\Delta}_t\|=\hat{M}_{t+1} - \hat{M}_t$, hence:
\begin{align*}
    \sumtT \norm{\Delta_t - \hat{\Delta}_t}^2
    = \sumtT \brac{\hat{M}_{t+1} - \hat{M}_t}^2
    \le \brac{\sumtT \brac{\hat{M}_{t+1} - \hat{M}_t} }^2
    = \brac{\hat{M}_{T+1} - \hat{M}_1}^2.
\end{align*}
Therefore, we conclude that:
\begin{align*}
    \sumtT\inner{g_t, w_t - u}
    &\le \psi_{T+1}(u) + \sumtT \underbrace{\inner{\hat \Delta_t, \wt-\wtpp}- \cD_{\psi_t}(\wtpp,\wt) - (\psi_\tpp - \psi_t)(\wtpp)}_{\eqcolon\hat{\delta}_t} \\
    &\qquad + \frac{\beta}{2}\norm{u}^2 + \frac{1}{2\beta}\brac{\hat{M}_{T+1} - \hat{M}_1}^2.
\end{align*}
To bound $\sumtT\hat{\delta}_t$, we apply~\pref{lemma:pf-regularizer} by defining $\Vbar_t = 4\hat{M}_t^2 + \sum_{s=1}^{t-1}\|\hat{\Delta}_s\|^2$, let $\alpha_1\ge \ldots\ge \alpha_T$ be a non-increasing sequence, and define $\psi_t(w) = 3\int_{0}^{\norm{w}} \min_{\eta\le 1/\hat{M}_t}\mbr{ \frac{\log(x/\alpha_t + 1)}{\eta} + \eta\Vbar_t } \mathrm{d}x$,
then
\begin{align*}
    \sumtT \hat{\delta}_t \le \sumtT \frac{2\alpha_t\norm{\hat{\Delta}_t}^2}{\sqrt{\Vbar_t}}.
\end{align*}
We conclude that:
\begin{align*}
    \sumtT\inner{g_t, w_t - u}
    &\le \psi_{T+1}(u) + \sumtT \hat{\delta}_t + \frac{\beta}{2}\norm{u}^2 + \frac{1}{2\beta}\brac{\hat{M}_{T+1} - \hat{M}_1}^2 \\
    &\le \psi_{T+1}(u) + \sumtT \frac{2\alpha_t\norm{\hat{\Delta}_t}^2}{\sqrt{\Vbar_t}} + \frac{\beta}{2}\norm{u}^2 + \frac{1}{2\beta}\brac{\hat{M}_{T+1} - \hat{M}_1}^2 \\
    &\le 6\norm{u}\max\bbr{ \sqrt{\bar{V}_{T+1}\log(\norm{u}/\alpha_{T+1} + 1)}, \hat{M}_{T+1}\log(\norm{u}/\alpha_{T+1} + 1) } \\
    &\qquad + \sumtT \frac{2\alpha_t\norm{\hat{\Delta}_t}^2}{\sqrt{\Vbar_t}} + \frac{\beta}{2}\norm{u}^2 + \frac{1}{2\beta}\brac{\hat{M}_{T+1} - \hat{M}_1}^2,
\end{align*}
where the last inequality uses~\pref{lemma:pf-regularizer}. Finally, to apply~\pref{lemma:tuning-sum}, we define $\alpha_t = \frac{\epsilon}{\sqrt{B_t} \log^2(B_t)}$ with $\smash{B_t=4+\sum_{i=1}^{t-1}\|\hat{\Delta}_i\|^2/\hat{M}_i^2}$, which yields
\begin{align*}
    \sumtT\inner{g_t, w_t - u}
    &\le 6\norm{u}\max\left\{ \sqrt{\bar{V}_{T+1}\log\brac{\frac{\norm{u}\sqrt{B_{T+1}}\log^2(B_{T+1})}{\epsilon} + 1}} , \right. \\
    &\qquad\qquad\qquad\quad \left. \hat{M}_{T+1}\log\brac{\frac{\norm{u}\sqrt{B_{T+1}}\log^2(B_{T+1})}{\epsilon} + 1} \right\} \\
    &\qquad + 8\epsilon\hat{M}_T + \frac{\beta}{2}\norm{u}^2 + \frac{1}{2\beta}\brac{\hat{M}_{T+1} - \hat{M}_1}^2.
\end{align*}
Finally, the stated bound follows by bounding $B_t \le 3+t$ (since $\|\hat{\Delta}_t\|\le\hat{M}_t$ by definition), and that $0<\hat{M}_{T+1} \le \max\{\gamma, 2G\} \le \gamma + 2G$, then $(\hat{M}_{T+1} - \hat{M}_1)^2 = (\hat{M}_{T+1} - \gamma)^2 \le 4G^2$.

The second regret bound is by applying~\pref{thm:smooth-optimistic}, for which we have
  \begin{align*}
    A_T(u) &= O\sbr{ \norm{u}(\gamma + G)\sbr{\log_+ \frac{\norm{u}\sqrt{T}}{\epsilon}} + \epsilon (\gamma + G) + \frac{\gamma}{\epsilon}\norm{u}^2 + \frac{\epsilon}{\gamma}G^2 }, \\
    B_T(u) &= O\sbr{ \norm{u}\sqrt{\log_+ \sbr{\frac{\norm{u}\sqrt{T}}{\epsilon}}} }.
  \end{align*}
  Plugging these into~\pref{thm:smooth-optimistic} yields the desired result.
\end{proof}

\subsection{Efficiency of the Virtual Clipping Algorithm}
\label{app:efficiency}

In this section we provide the detailed discussion of the computational efficiency of the virtual clipping algorithm, \pref{alg:simple-optimistic}, proposed in \pref{subsec:simple-optimistic}.
We first recall \pref{prop:efficiency}, re-stated below for convenience, which we prove later in this section.

\Efficiency*
Using \pref{prop:efficiency} and denoting $\theta_t\coloneq\norm{\gt-h_t+h_\tpp}$, we have that the norm of $\wtpp$ is no more than a factor of $\theta_t/\beta$ away from the norm of $\wt$. This means that we could naively apply a binary search over a bracket of 
$[\norm{\wt} - \theta_t/\beta, \norm{\wt}+\theta_t/\beta]$ to compute $\norm{\wtpp}$ 
up to $\epsilon$ accuracy using $O(\Log{2\theta_t/\beta\epsilon})$
iterations. Thus, we can solve $\norm{\wtpp}$ up to $\epsilon = \frac{2\theta_t}{\beta T^k}=O(1/T^k)$ accuracy for any $k\ge 1$ while still matching the efficiency of \citet{wang2025parameter}, which would already be sufficient for most applications. In practice, 
this naive estimate above could potentially be significantly reduced by applying a more sophisticated solver than the simple binary search suggested above (e.g., by applying Newton's method).

\begin{proof}[Proof of~\pref{prop:efficiency}]
We first present a direct derivation of the update for \pref{alg:simple-optimistic}.
Denote
$\hat \psi(x;\alpha,V,M)=3\int_0^x\min_{\eta\le 1/M}\mbr{ \frac{\ln(x/\alpha+ 1)}{\eta} + \eta V } \mathrm{d}x$ and 
$\psi(w;\alpha,V,M)=\hat\psi(\norm{w};\alpha,V,M)$.
On round $t$,
the algorithm chooses a hybrid regularizer $\Psi_{t+1}(w) = \psi^{\textsc{pf}}\big(w;\alpha_\tpp,\Vbar_{t+1},\hat{M}_{t+1}\big) + \frac{\beta}{2}\norm{w}^2$ (we use simplified notations of $\alpha,V,M$ in the following) and
updates using $w_{t+1} = \argmin_{w\in\R^d} \inner{h_{t+1} + \sum_{s=1}^{t}g_s, w} + \Psi_{t+1}(w)$.
By the first-order optimality condition of $w_{t+1}$, we have
\begin{align*}
    \mathbf{0} &=
    h_{t+1} + \sum_{s=1}^{t}g_s + \grad \psi^{\textsc{pf}}\big(w_{t+1};\alpha,V,M\big) + \beta w_{t+1} \\
    &=
    \begin{cases}
    h_{t+1} + \sum_{s=1}^{t}g_s + 6\frac{w_{t+1}}{\norm{w_{t+1}}}\sqrt{V\ln(\norm{w_{t+1}}/\alpha + 1)} + \beta w_{t+1},&\!\!\text{if }\sqrt{\frac{\ln(\norm{w_{t+1}}/\alpha + 1)}{V}} \le \frac{1}{M}, \\
    h_{t+1} + \sum_{s=1}^{t}g_s + 3\frac{w_{t+1}}{\norm{w_{t+1}}}\sbr{M\ln(\norm{w_{t+1}}/\alpha + 1) + \frac{V}{M}} + \beta w_{t+1},&\!\!\text{otherwise}.
    \end{cases}
\end{align*}
This immediately implies that $w_{t+1}$ is in the direction of $-h_{t+1} - \sum_{s=1}^{t}g_s$.
Taking the norm on both sides of the above equation, we have
\begin{align*}
    \norm{h_{t+1} + \sum_{s=1}^{t}g_s} &=
    \begin{cases}
    6\sqrt{V\ln(\norm{w_{t+1}}/\alpha + 1)} + \beta \norm{w_{t+1}}, & \text{if } \sqrt{\frac{\ln(\norm{w_{t+1}}/\alpha + 1)}{V}} \le \frac{1}{M}, \\
    3\sbr{M\ln(\norm{w_{t+1}}/\alpha + 1) + \frac{V}{M}} + \beta \norm{w_{t+1}}, & \text{otherwise}.
    \end{cases}
\end{align*}
This equation has no closed-form solution for $\norm{\wtpp}$ in general, but
an approximate solution can be found via binary search over a small bracket around $\norm{\wt}$,
as shown in \pref{prop:bracket} applied with $\hat\psi_t(\norm{w})=\hat\psi_t^{\textsc{PF}}(\norm{w};\alpha_t,\bar V_t,\hat M_t)$, which can easily be seen to satisfy the required conditions 
$\hat\psi_t'(x)\ge 0$ and $\psi_t'(x)\le \psi_\tpp'(x)$ for non-increasing sequence $(\alpha_t)_t$.
\end{proof}

\begin{proposition}\label{prop:bracket}
    Let $\hat\psi_1,\hat\psi_2,\ldots$ be an sequence of convex regularizers satisfying $\hat\psi_t'(x)\ge 0$ and $\hat\psi_t'(x)\le \hat\psi_\tpp'(x)$ for any $t$ and any $x\ge 0$.
    For each $t$ define $\psi_t(w)=\hat\psi_t(\norm{w})$ and $\Psi_t(w)=\psi_t(w) + \frac{\beta}{2}\norm{\w}^2$ for $\beta>0$, and suppose
    that on each round we play 
    \begin{align*}
        \wt=\argmin_{w\in\R^d}\inner{h_t+\sum_{s=1}^\tmm \gs, w}+\Psi_t(w).
    \end{align*}
    Then for any $t$, it holds that
    \begin{align*}
        \abs{\norm{\wtpp}-\norm{\wt}}\le \frac{\norm{\gt-h_t+h_\tpp}}{\beta}.
    \end{align*}
\end{proposition}
\begin{proof}
Observe that for any $t$, we have via the first-order optimality condition for the optimistic FTRL update that
\begin{align*}
    \wtpp &= \argmin_{w\in\R^d}\inner{h_\tpp+\sum_{s=1}^t\gs, \w}+\Psi_\tpp(\w)\\
    &\implies \grad\Psi_\tpp(\wtpp) = -h_\tpp-\sum_{s=1}^t \gs,
\end{align*}
and moreover, since this holds for any $t$, we also have
$\grad\Psi_t(\wt)=-h_t-\sum_{s=1}^\tmm\gs$, so the previous display can be 
written as
\begin{align*}
    \grad\Psi_\tpp(\wtpp)&= \grad\Psi_t(\wt) -(\gt-h_t + h_\tpp).
\end{align*}
Therefore, by denoting $\Psi_t(w) = \hat{\Psi}_t(\norm{w})$ and $\psi_t(w)=\hat{\psi}_t(\norm{w})$, we have
\begin{align*}
    \hat{\Psi}_\tpp'(\norm{\wtpp})&={\hat{\psi}}_{t+1} {}'(\norm{\wtpp}) +\beta\norm{\wtpp}\\
    &=\norm{\grad\Psi_t(\wt) - (\gt-h_t+h_\tpp)}
    \le \norm{\grad\Psi_t(\wt)}+\norm{\gt-h_t+h_\tpp}\\
    &= \hat{\Psi}_t'(\norm{\wt})+\norm{\gt-h_t+h_\tpp}
    \le \hat{\Psi}_\tpp'(\norm{\wt})+\norm{\gt-h_t+h_\tpp},
\end{align*}
where 
we've used the facts that $\hat{\Psi}_t'(x)\ge 0$  and  
$\hat{\Psi}_t'(x)\le \hat{\Psi}_\tpp'(x)$ for all $x\ge 0$.
From this, we have in particular that
\begin{align*}
    \abs{\hat{\Psi}_\tpp'(\norm{\wtpp}) -\hat{\Psi}_\tpp'(\norm{\wt}) }\le \norm{\gt-h_t+h_\tpp}:=\theta_t.
\end{align*}
We can use this to derive a bound on how far $\norm{\wtpp}$ is from $\norm{\wt}$
by studying the inverse of $\hat{\Psi}'_\tpp$, since $\norm{\w}=(\hat{\Psi}'_\tpp)^{-1}(\hat{\Psi}_\tpp'(\norm{\w}))$. 

Now, since
$\hat \psi_\tpp$ is convex, we have $\hat{\psi}_\tpp{}''(x)\ge 0$ for all $x$
and therefore,
$\hat{\Psi}_\tpp''(x)= \hat{\psi}_\tpp{}''(x) + \beta \ge \beta>0$ for all $x$. 
Therefore, $\hat{\Psi}_\tpp'(x)$ is $\beta$-strongly monotone, 
which implies that it has a unique inverse $(\hat{\Psi}_\tpp')^{-1}$, and moreover, that $(\hat{\Psi}_\tpp')^{-1}$ is $1/\beta$-Lipschitz. Hence, 
\begin{align*}
     \abs{\norm{\wtpp} -\norm{\wt}} &= \abs{(\hat{\Psi}_\tpp')^{-1}\brac{\hat{\Psi}_\tpp'(\norm{\wtpp})}-(\hat{\Psi}_\tpp')^{-1}\brac{\hat{\Psi}_\tpp'(\norm{\wt})}}\\
     &\le \frac{\abs{\hat{\Psi}_\tpp'(\norm{\wtpp})-\hat{\Psi}_\tpp'(\norm{\wt})}}{\beta}\le \frac{\norm{\gt-h_t+h_{\tpp}}}{\beta}=\frac{\theta_t}{\beta}.
\end{align*}
\end{proof}

\subsection{Proof of Theorem~\ref{thm:efficient-optimistic}}
\label{app:proof-efficient-optimistic}

To make the paper self-contained, we first restate~\citet[Theorem 1]{cutkosky2024fully} as a lemma here. To be consistent with this paper, we replace their notation $\gamma$ with $\gamma'/\epsilon'$, replace $h_1$ with $\gamma'$, and replace notation $\epsilon$ with $\epsilon'$.

\begin{lemma}[Theorem 1 of~\citet{cutkosky2024fully}]
  \label{lemma:cutkosky2024fully}
  There exists an online learning algorithm $\cA$
  with any inputs $\epsilon'>0$ and $\gamma'>0$,
  that runs in $O(d)$ time per iteration,
  such that for any sequence $g_1,\ldots,g_T$, and for any $u\in\R^d$, guarantees
  \begin{align*}
    \sumtT\inner{g_t, w_t - u}
    &\le
    O\Bigg( \norm{u}\sqrt{\sumtT\norm{g_t}^2 \sbr{\log_+\tfrac{G\norm{u}\sqrt{T}}{\epsilon'\gamma'}}} + G\norm{u}\sbr{\log_+ \tfrac{G\norm{u}\sqrt{T}}{\epsilon'\gamma'}} \\
    &\qquad\qquad + \tfrac{\epsilon' G^2}{\gamma'}\sbr{\log_+ \tfrac{G}{\gamma'}} + \tfrac{\gamma'\norm{u}^2}{\epsilon'}\sbr{\log_+ \tfrac{\norm{u}}{\epsilon'}} + \epsilon' G + \epsilon'\gamma' 
      \Bigg),
  \end{align*}
  where $G = \max\{\gamma', \max_{t\in[T]}\norm{g_t}\}$.
\end{lemma}

\noindent
We restate~\pref{thm:efficient-optimistic} here with more detailed logarithmic factors.
\begin{manualtheorem}{\ref{thm:efficient-optimistic}}[Full Version]
  For any $u\in\R^d$,
  Algorithm~\ref{alg:efficient-optimistic} 
  guarantees
\begin{align*}
    \reg_T(u)
    &\le O\Bigg( \epsilon G + \epsilon \gamma + \frac{\epsilon G^2}{\gamma}\brac{\log_+ \frac{G}{\gamma}} + \frac{\gamma\norm{u}^2}{\epsilon}\brac{\log_+ \frac{\norm{u}}{\epsilon} }+ \frac{\epsilon G^4}{\gamma^3}\brac{\log_+ \frac{G}{\gamma}}  \\
    &\qquad\qquad + \norm{u} G \brac{\log_+ \frac{\norm{u}G\sqrt{T}}{\epsilon\gamma} }^{\frac{3}{2}} + \frac{\gamma^3\norm{u}^2}{\epsilon G^2}\brac{\log_+ \frac{\norm{u}G\sqrt{T}}{\epsilon\gamma} } \brac{\log_+ \frac{\gamma\norm{u}}{\epsilon G} } \\
    &\qquad\qquad + \norm{u}\sqrt{ V_T(u) \brac{\log_+ \frac{\norm{u}G\sqrt{T}}{\epsilon\gamma} } } + L\norm{u}^2 \brac{\log_+ \frac{\norm{u}G\sqrt{T}}{\epsilon\gamma} } \Bigg).
\end{align*}
  Moreover, the algorithm admits an efficient closed-form update with $O(d)$ time per iteration.
\end{manualtheorem}

\begin{proof}
We will apply our black-box reduction (\pref{prop:black-box-optimism-partial}) with the algorithm of~\citet[Theorem 1]{cutkosky2024fully}, provided in~\pref{lemma:cutkosky2024fully} above, for both $\cA_x$ and $\cA_y$. By~\pref{lemma:cutkosky2024fully}, the algorithm $\cA_x$ with inputs $\epsilon'=\epsilon$ and $\gamma'=\gamma$ satisfies the conditions of~\pref{prop:black-box-optimism-partial} with 
\begin{align}
    A_T^{\cA_x}(u) &= O\Bigg( \epsilon G + \epsilon \gamma + \frac{\epsilon G^2}{\gamma}\brac{\log_+ \frac{G}{\gamma}} \nonumber\\
    &\qquad
    + \frac{\gamma\norm{u}^2}{\epsilon}\brac{\log_+ \frac{\norm{u}}{\epsilon} }+ \norm{u}G\brac{\log_+ \frac{\norm{u}G\sqrt{T}}{\epsilon\gamma} } \Bigg), \label{eq:black-box-AAx}\\
    B_T^{\cA_x}(u) &= O\brac{ \norm{u}\sqrt{\log_+{ \frac{\norm{u}G\sqrt{T}}{\epsilon\gamma} }} },\label{eq:black-box-BAx}
\end{align}
where $G$ is the Lipschitz constant. Likewise, for $\cA_y$, given any $\epsilon'>0$ and $\gamma'>0$ (to be specified later) we have:
\begin{align*}
    A_T^{\cA_y}(\cmpy) &= O\Bigg( \epsilon' GH + \epsilon' \gamma' + \frac{\epsilon'G^2H^2}{\gamma'}\brac{\log_+ \frac{GH}{\gamma'}}\\
    &\qquad\qquad {} + \frac{\gamma'|\cmpy|^2}{\epsilon'}\brac{\log_+ \frac{|\cmpy|}{\epsilon'} }+ |\cmpy|GH\brac{\log_+ \frac{|\cmpy|GH\sqrt{T}}{\epsilon'\gamma'} } \Bigg), \\
    B_T^{\cA_y}(\cmpy) &= O\brac{ |\cmpy|\sqrt{\log_+ \frac{|\cmpy|GH\sqrt{T}}{\epsilon'\gamma'} } } = O\brac{ |\cmpy| \lambda_T(\cmpy) },
\end{align*}
where $H \define \max_{t\in[T]} \norm{h_t}$ and $\lambda_T(\cmpy)\define 1 + \sqrt{\log_+\brac{ {|\cmpy|GH\sqrt{T}} / (\epsilon'\gamma') }}$. Applying \pref{prop:black-box-optimism-partial} we have
\begin{align}
    \sumtT \inner{g_t, w_t - u}
    &\le
      A_{T}^{\cA_{x}}(\cmp)+A_{T}^{\cA_{y}}(\cmpy) + 2B_{T}^{\cA_{x}}(\cmp)
      \sqrt{H^{2}\lambda_{T}(\cmpy)^{2}+\half \sumtT \norm{\gt-h_{t}}^{2}}.\label{eq:black-box-1}
\end{align}
where 
$\cmpy=B_{T}^{\cA_{x}}(\cmp)/H$.
Expanding $\mathring{y}$ in $A_{T}^{\cA_{y}}(\cmpy)$ and setting $\epsilon'=\epsilon/\gamma$, and $\gamma'=\gamma^2$, we have
\begin{align}
    A_{T}^{\cA_{y}}(\cmpy) &= A_{T}^{\cA_{y}}({B_{T}^{\cA_{x}}(\cmp)}/{H}) \nonumber\\
    &= O\Bigg( \epsilon' GH + \epsilon' \gamma' + \frac{\epsilon'G^2H^2}{\gamma'}\brac{\log_+ \frac{GH}{\gamma'}} \nonumber\\
    &\qquad  {} + \frac{\gamma'B_{T}^{\cA_{x}}(\cmp)^2}{\epsilon'H^2}\brac{ \log_+\frac{B_{T}^{\cA_{x}}(\cmp)}{\epsilon'H} }+ B_{T}^{\cA_{x}}(\cmp)G\brac{\log_+ \frac{B_{T}^{\cA_{x}}(\cmp)G\sqrt{T}}{\epsilon'\gamma'} } \Bigg) \nonumber\\
    &= O\Bigg( \epsilon' GH + \epsilon' \gamma' + \frac{\epsilon'G^2H^2}{\gamma'}\brac{\log_+ \frac{GH}{\gamma'}} \nonumber\\
    &\qquad  {} + \frac{\gamma'\norm{u}^2}{\epsilon'H^2}\brac{\log_+ \frac{\norm{u}G\sqrt{T}}{\epsilon\gamma} } \brac{\log_+ \frac{\norm{u}}{\epsilon'H} } \nonumber\\
    &\qquad+ \norm{u}G \sqrt{\log_+ \frac{\norm{u}G\sqrt{T}}{\epsilon\gamma} } \brac{\log_+ \frac{\norm{u}G\sqrt{T}}{\epsilon'\gamma'} } \Bigg) \nonumber\\
    &= O\Bigg( \frac{\epsilon GH}{\gamma} + \epsilon \gamma + \frac{\epsilon G^2H^2}{\gamma^3}\brac{\log_+ \frac{GH}{\gamma^2}} \nonumber\\
    &\qquad {} + \frac{\gamma^3\norm{u}^2}{\epsilon H^2}\brac{\log_+ \frac{\norm{u}G\sqrt{T}}{\epsilon\gamma} } \brac{\log_+ \frac{\gamma\norm{u}}{\epsilon H} } \nonumber\\
    &\qquad+ \norm{u} G \brac{\log_+ \frac{\norm{u}G\sqrt{T}}{\epsilon\gamma} }^{\frac{3}{2}} \Bigg),\label{eq:black-box-AAy}
\end{align}
and likewise, expanding $\cmpy$ in $\lambda_T(\cmpy)$ yields 
\begin{align}
    \lambda_T(\cmpy)^2=\lambda_T(B_{T}^{\cA_{x}}(\cmp)/H)^2 =1 + \log_+ \frac{B_{T}^{\cA_{x}}(\cmp)G\sqrt{T}}{\epsilon\gamma}  = O\brac{\log_+\frac{ \norm{u}G\sqrt{T}}{\epsilon\gamma} }.\label{eq:black-box-lambda}
\end{align}
Plugging~\Cref{eq:black-box-AAx,eq:black-box-BAx,eq:black-box-AAy,eq:black-box-lambda} back into \Cref{eq:black-box-1}, we have
\begin{align*}
    &\sumtT \inner{g_t, w_t - u}
    \le
      A_{T}^{\cA_{x}}(\cmp)+A_{T}^{\cA_{y}}(\cmpy) + 2B_{T}^{\cA_{x}}(\cmp)
      \sqrt{H^{2}\lambda_{T}(\cmpy)^{2}+\half \sumtT \norm{\gt-h_{t}}^{2}} \\
    \le{} & O\Bigg( \epsilon G + \epsilon \gamma + \frac{\epsilon GH}{\gamma} + \frac{\epsilon G^2}{\gamma}\brac{\log_+ \frac{G}{\gamma}} + \frac{\gamma\norm{u}^2}{\epsilon}\brac{\log_+ \frac{\norm{u}}{\epsilon} }+ \frac{\epsilon G^2H^2}{\gamma^3}\brac{\log_+ \frac{GH}{\gamma^2}} \\
    &\qquad + \norm{u} H \brac{\log_+ \frac{\norm{u}G\sqrt{T}}{\epsilon\gamma} } + \frac{\gamma^3\norm{u}^2}{\epsilon H^2}\brac{\log_+ \frac{\norm{u}G\sqrt{T}}{\epsilon\gamma} } \brac{\log_+ \frac{\gamma\norm{u}}{\epsilon H} } \\
    &\qquad + \norm{u} G \brac{\log_+ \frac{\norm{u}G\sqrt{T}}{\epsilon\gamma} }^{\frac{3}{2}} + \norm{u}\sqrt{ \Vbar_T \brac{\log_+ \frac{\norm{u}G\sqrt{T}}{\epsilon\gamma} } } \Bigg),
\end{align*}
where we define $\Vbar_T = \sumtT \norm{g_t - h_t}^2$. Finally by setting $h_t = g_{t-1}$ and hence $H=G$, then apply~\pref{thm:smooth-optimistic}, we conclude that:
\begin{align*}
    &\sumtT f_t(w_t) - f_t(u) = \sumtT \inner{g_t, w_t - u} - \cD_{f_t}(u, w_t) \\
    \le{} & O\Bigg( \epsilon G + \epsilon \gamma + \frac{\epsilon G^2}{\gamma}\brac{\log_+ \frac{G}{\gamma}} + \frac{\gamma\norm{u}^2}{\epsilon}\brac{\log_+ \frac{\norm{u}}{\epsilon} }+ \frac{\epsilon G^4}{\gamma^3}\brac{\log_+ \frac{G}{\gamma}}  \\
    &\qquad + \norm{u} G \brac{\log_+ \frac{\norm{u}G\sqrt{T}}{\epsilon\gamma} }^{\frac{3}{2}} + \frac{\gamma^3\norm{u}^2}{\epsilon G^2}\brac{\log_+ \frac{\norm{u}G\sqrt{T}}{\epsilon\gamma} } \brac{\log_+ \frac{\gamma\norm{u}}{\epsilon G} } \\
    &\qquad + \norm{u}\sqrt{ V_T(u) \brac{\log_+ \frac{\norm{u}G\sqrt{T}}{\epsilon\gamma} } } + L\norm{u}^2 \brac{\log_+ \frac{\norm{u}G\sqrt{T}}{\epsilon\gamma} } \Bigg),
\end{align*}
where $V_T(u) = \sum_{t=2}^{T} \norm{\grad f_{t}(\cmp)-\grad f_{\tmm}(\cmp)}^{2}$.
\end{proof}

\begin{algorithm2e}[t]
\caption{Optimistic Reduction~\citep{cutkosky2019combining}}
  \label{alg:optimistic-reduction}
  \textbf{Input: }Online Learning algorithms $\cA_{x}$ defined on $\R^{d}$ and $\cA_{y}$ defined
    on $\R$.\\
  \For{$t=1:T$}{
    Observe $h_{t}\in\R^{d}$\\
    Get $\xt\in\R^{d}$ from $\cA_{x}$ and $\yt\in\R$ from $\cA_{y}$\\
    Play $\wt = \xt - \yt h_{t}$ and observe $\gt =\grad f_{t}(\wt)$\\
    Pass $\gt$ to $\cA_{x}$ as the $t^{\text{th}}$ subgradient\\
    Pass $-\inner{\gt,h_{t}}$ to $\cA_{y}$ as the $t^{\text{th}}$ subgradient
  }
\end{algorithm2e}

\subsection{Proof of Proposition~\ref{prop:black-box-optimism-partial}}
\label{app:black-box-optimism-partial}

We provide the Optimistic Reduction~\citep{cutkosky2019combining} in~\pref{alg:optimistic-reduction} for self-containment, and give the following proof for the refined reduction in~\pref{prop:black-box-optimism-partial}. The result is re-stated below for convenience.
\BlackBoxOptimismPartial*
\begin{proof}
  For $\wt=\xt-\yt h_{t}$ and $\gt =\grad f_{t}(\wt)$, we have
  \begin{align*}
      \sumtT \inner{\gt, \wt-\cmp}
      &=
      \sumtT \inner{\gt, \xt-\cmp} + \sumtT \inner{-\gt,h_{t}}\yt\\
    &=
      \reg_{T}^{\cA_{x}}(\cmp)+\sumtT \brac{\inner{-\gt,h_{t}}\yt - \inner{-\gt,h_{t}}\cmpy} + \cmpy\sumtT \inner{- \gt,h_{t}}\\
    &=
      \reg_{T}^{\cA_{x}}(\cmp)+\reg_{T}^{\cA_{y}}(\cmpy)+\cmpy\sumtT \inner{-\gt, h_{t}},
  \end{align*}
  and moreover, using the fact that
  $-\inner{\gt, h_{t}}= \half\norm{\gt-h_{t}}^{2}-\half\norm{h_{t}}^{2}-\half\norm{\gt}^{2}$, we
  have
  \begin{align*}
    \sumtT \inner{\gt, \wt-\cmp}
    &\le
      \reg_{T}^{\cA_{x}}(\cmp) + \reg_{T}^{\cA_{y}}(\cmpy)+\frac{\cmpy}{2}\brac{\bar{V}_{T} - \sumtT \norm{\gt}^{2}},
  \end{align*}
  where we've defined $\bar V_{T}:=\sbrac{\sumtT\norm{\gt-h_{t}}^{2}-\norm{h_{t}}^{2}}_{+}$ and $[x]_+ = \max\{x, 0\}$.
  Applying the regret guarantees of $\cA_{x}$ and $\cA_{y}$, we have
  \begin{align*}
    \sumtT \inner{\gt, \wt-\cmp}
    &\le
      A_{T}^{\cA_{x}}(\cmp)+B_{T}^{\cA_{x}}(\cmp)\sqrt{\sumtT \norm{\gt}^{2}}
    +
      A_{T}^{\cA_{y}}(\cmpy)+B_{T}^{\cA_{y}}(\cmpy)\sqrt{\sumtT \inner{\gt, h_{t}}^{2}}
      \\
    &\qquad+
      \frac{\cmpy}{2}\brac{\bar V_{T}-\sumtT \norm{\gt}^{2}} \\
    &=
      A_{T}^{\cA_{x}}(\cmp)+A_{T}^{\cA_{y}}(\cmpy)+\frac{\cmpy}{2}\bar{V}_{T}\\
    &\qquad
      +\sbrac{B_{T}^{\cA_{x}}(\cmp)+B_{T}^{\cA_{y}}(\cmpy)\max_{t\in[T]}\norm{h_{t}}}\sqrt{\sumtT \norm{\gt}^{2}}-\frac{\cmpy}{2}\sumtT \norm{\gt}^{2}\\
    &\le
      A_{T}^{\cA_{x}}(\cmp)+A_{T}^{\cA_{y}}(\cmpy)+\frac{\cmpy}{2}\bar{V}_{T}\\
    &\qquad
      +\sup_{X\ge 0}\Set{ \sbrac{B_{T}^{\cA_{x}}(\cmp)+B_{T}^{\cA_{y}}(\cmpy)\max_{t\in[T]}\norm{h_{t}}}X-\frac{\cmpy}{2}X^{2}}\\
    &\le
      A_{T}^{\cA_{x}}(\cmp)+A_{T}^{\cA_{y}}(\cmpy)+\frac{\cmpy}{2}\bar{V}_{T}
      +\frac{\sbrac{B_{T}^{\cA_{x}}(\cmp)+B_{T}^{\cA_{y}}(\cmpy)H_{T}}^{2}}{2\cmpy}\\
    &\le
      A_{T}^{\cA_{x}}(\cmp)+A_{T}^{\cA_{y}}(\cmpy)+\frac{\cmpy}{2}\bar{V}_{T}
      +\frac{B_{T}^{\cA_{x}}(\cmp)^{2}+B_{T}^{\cA_{y}}(\cmpy)^{2}H_{T}^{2}}{\cmpy},
  \end{align*}
  where we've used $aX-bX^{2}\le a^{2}/4b$ and defined
  $H_{T}=\max_{t}\norm{h_{t}}$.
  For
  $B_{T}^{\cA_{y}}(\cmpy)\le \cmpy \lambda_{T}(\cmpy)$, we have
  \begin{align*}
    \sumtT \inner{\gt, \wt-\cmp}
    &\le
      A_{T}^{\cA_{x}}(\cmp)+A_{T}^{\cA_{y}}(\cmpy)+\cmpy\sbrac{H_{T}^{2}\lambda_{T}(\cmpy)^{2}+\half \bar V_{T}}+\frac{B_{T}^{\cA_{x}}(\cmp)^{2}}{\cmpy}.
  \end{align*}
  Since this holds for any $\cmpy\in\R_+$, we may take
  $\cmpy = B_{T}^{\cA_{x}}(\cmp)/\sqrt{{H_{T}^{2}\lambda_{T}^{2}(B_{T}^{\cA_{x}}(\cmp)/H_{T})+\half \bar V_{T}}}\le B_{T}^{\cA_{x}}(\cmp)/H_{T}$,
  and so for monotonically increasing $A_{T}^{\cA_{y}}(\cmpy)$ and
  $\lambda_{T}(\cmpy)$, we have
  \begin{align*}
    \sumtT \inner{\gt, \wt-\cmp}
    &\le
      A_{T}^{\cA_{x}}(\cmp)+A_{T}^{\cA_{y}}\sbr{{B_{T}^{\cA_{x}}(\cmp)}/{H_{T}}} + 2B_{T}^{\cA_{x}}(\cmp)
      \sqrt{{H_{T}^{2}{\lambda_{T}\sbr{{B_{T}^{\cA_{x}}(\cmp)}/{H_{T}}}^{2}}+\half \bar{V}_{T}}}.
  \end{align*}
  The stated result defines $\cmpy=B_{T}^{\cA_{x}}(\cmp)/H_{T}$
  for brevity.
\end{proof}

\section{Proofs for Section~\ref{sec:applications}}
\label{app:applications}

\subsection{Proof of Theorem~\ref{thm:dynamic}}
\label{app:proof-dynamic}

We re-state \pref{thm:dynamic} here with more detailed logarithmic factors.

\begin{manualtheorem}{\ref{thm:dynamic}}[Full Version]
    \label{thm:dynamic-full}
    For any sequence $u_1,\ldots,u_T\in\R^d$, Algorithm~\ref{alg:dynamic} guarantees
\begin{align*}
    \reg_T(u_{1:T}) &
    \le{} O \Bigg(
        \sqrt{\brac{M^2 + MP_T} V_T(u_{1:T}) \brac{ \log_+ \tfrac{MG\sqrt{T}}{\epsilon \gamma} } } + L(M^2+MP_T)\brac{ \log_+ \tfrac{MG\sqrt{T}}{\epsilon \gamma} } \\
    &\qquad\qquad + G\sqrt{M^2+MP_T} + P_T G  + \gamma P_T \\
    &\qquad\qquad + \epsilon G + \gamma M + \epsilon \gamma + \tfrac{\epsilon G^2}{\gamma}\brac{\log_+ \tfrac{G}{\gamma}} + \tfrac{\gamma M^2}{\epsilon}\brac{\log_+ \tfrac{M}{\epsilon} }+ \tfrac{\epsilon G^4}{\gamma^3}\brac{\log_+ \tfrac{G}{\gamma}} \\
    &\qquad\qquad + M G \brac{\log_+ \tfrac{MG\sqrt{T}}{\epsilon\gamma} }^{\frac{3}{2}} + MG\sbr{\log_+ \tfrac{GT}{\gamma}} \\
    &\qquad\qquad+ \tfrac{\gamma^3 M^2}{\epsilon G^2}\brac{\log_+ \tfrac{MG\sqrt{T}}{\epsilon\gamma} } \brac{\log_+ \tfrac{\gamma M}{\epsilon G} } \Bigg),
\end{align*}
where $M \define \max_{t\in[T]}\norm{u_t}$. Moreover, the algorithm runs in $O(d\log t)$ time on iteration $t$.
\end{manualtheorem}

\begin{proof}
Following~\citet[Theorem 2]{COLT'18:black-box-reduction} and~\citet[Lemma 10]{jacobsen2022parameter}, the linearized dynamic regret is bounded as
\begin{align*}
    \sumtT \inner{\gt, \wt-\cmp_t}
    &= \sumtT \inner{\gt, x_t}y_t -\inner{\gt,\cmp_t}\\
    &= \sumtT \inner{\gt, x_t}y_t - \inner{\gt, x_t} M + \sumtT \inner{\gt, x_t}M - \inner{\gt, \cmp_t}\\
    &=
    \reg_T^\onedim(M) + M\sumtT \inner{\gt,x_t-\frac{\cmp_t}{M}}\\
    &=
    \reg_T^{\onedim}(M) + M \reg_T^{\cB}(\cmp_{1:T}/M),
\end{align*}
where we define $M\define \max_{t\in[T]}\norm{u_t}$, $\reg_T^\onedim(M)$ denotes the static regret of an algorithm $\cA_\onedim$ which receives optimistic hint $\inner{h_t, x_t}$, then chooses $y_t\in\R$
and suffers losses $y\mapsto \inner{\gt,x_t}y$, while
$\reg_T^\cB(\cmp_{1:T}/M)$ denotes the dynamic regret of an algorithm $\cA_{\cB}$ receiving optimism $h_t$ and
choosing $x_t\in\cB = \{x\in\R^d:\norm{x}\le 1\}$ against losses $x\mapsto \inner{\gt,x}$.
Moreover, $\norm{h_t} \le G$ since we set $h_t=g_{t-1}$.

For $\cA^{\cB}$, we apply~\pref{thm:anytime-dynamic} which ensures the following guarantee:
\begin{align*}
    \reg_T^{\cB}(u_{1:T}/M) &\le O\Bigg( \sqrt{\brac{1 + P_T/M}\sumtT \norm{g_t - h_t}^2 } + G\sqrt{P_T/M} \\
    &\qquad+ G P_T/M + G\log_+\sbr{\tfrac{GT}{\gamma}} + \gamma(1+P_T/M) \Bigg).
\end{align*}
For $\cA^\onedim$, we apply proof in~\pref{thm:efficient-optimistic} with optimism $\inner{h_t,x_t}$ and gradient $\inner{g_t, x_t}$, hence:
\begin{align*}
    \reg_T^\onedim(M)
    &= \sumtT \inner{ g_t, x_t } (y_t - M) \\
    &\le O\brac{ M\sqrt{\sumtT \norm{\inner{g_t - h_t, x_t}}^2 \brac{\log_+ \frac{MG\sqrt{T}}{\epsilon\gamma}} } + \textsc{Const} } \\
    &\le O\brac{ M\sqrt{\sumtT \norm{g_t - h_t}^2 \brac{\log_+ \frac{MG\sqrt{T}}{\epsilon\gamma}} } + \textsc{Const}},
\end{align*}
where we use $\|x_t\|\le 1$ and denote constants by $\textsc{Const}$:
\begin{align*}
    \textsc{Const} &= \epsilon G + \epsilon \gamma + \frac{\epsilon G^2}{\gamma}\brac{\log_+ \frac{G}{\gamma}} + \frac{\gamma M^2}{\epsilon}\brac{\log_+ \frac{M}{\epsilon} }+ \frac{\epsilon G^4}{\gamma^3}\brac{\log_+ \frac{G}{\gamma}} \\
    &\qquad + M G \brac{\log_+ \frac{MG\sqrt{T}}{\epsilon\gamma} }^{\frac{3}{2}} + \frac{\gamma^3 M^2}{\epsilon G^2}\brac{\log_+ \frac{MG\sqrt{T}}{\epsilon\gamma} } \brac{\log_+ \frac{\gamma M}{\epsilon G} }.
\end{align*}
Combining everything together and applying~\pref{thm:smooth-optimistic}, we have:
\begin{align*}
    \reg_T(u_{1:T}) &
    \le{} O \Bigg(
        \sqrt{\brac{M^2 + MP_T} V_T(u_{1:T}) \brac{ \log_+ \frac{MG\sqrt{T}}{\epsilon \gamma} } } \\
        &\qquad\qquad+ L(M^2+MP_T)\brac{ \log_+ \frac{MG\sqrt{T}}{\epsilon \gamma} } + G\sqrt{M^2+MP_T} \\
    &\qquad\qquad + P_T G + MG\sbr{\log_+ \tfrac{GT}{\gamma}} + \gamma(D+P_T) + \textsc{Const} \Bigg).
\end{align*}
Finally, the computational efficiency is from~\pref{alg:efficient-optimistic} and~\pref{alg:anytime-dynamic}.
\end{proof}

\subsection{Anytime and Lipschitz-Adaptive Dynamic Regret Algorithm}
\label{app:anytime-dynamic}

\begin{algorithm2e}[t]
\caption{Anytime Dynamic Regret Minimization}
  \label{alg:anytime-dynamic}
  \textbf{Input: }Domain diameter $D$, $\gamma>0$.\\
  \textbf{Initialize: }Base learner number $n=1$, $h_1=\zeros$, $w_{1,1} = \zeros$, $p_1=1$, index set $\cI_1 = \emptyset$.\\
  \For{$t=1$ \textnormal{\textbf{to}} $T$}{
    Play $w_t = \sum_{i=1}^{n}p_{t,i}w_{t,i}$, then observe $g_t$ and hint $h_{t+1}$ \\
    Update $\cI_n \gets \cI_n \cup \{t\}$, define $\Vbar_{\cI_n} = \sum_{s\in \cI_n} \norm{g_s - h_s}^2$ and $\Vbar_t = \sum_{s=1}^{t}\norm{g_s - h_s}^2$\\
    \For{$i=1$ \textnormal{\textbf{to}} $n$}{
        Update $w_{t+1,i} = \Pi_{\cW}[ w_{t,i} - \eta_{t+1,i}(g_t - h_t + h_{t+1}) ]$ with $\eta_{t+1,i}=\frac{D 2^n}{\sqrt{\Vbar_t}}$
    }
    \If{$\Vbar_{\cI_n} / \gamma^2 > 2^n$}{
        $n \gets n+1$ \\
        Initialize $w_{t+1,n} = \zeros$, $p_{t+1} = \frac{1}{n}\mathbf{1}_n$, $\cI_n = \emptyset$
    }
    \Else{
        Define $\ell_t\in\R^n$ where $\ell_{t,i} = \inner{g_t, w_{t,i}}$ \\
        Define $m_{t+1}\in\R^n$ where $m_{t+1,i} = \inner{h_{t+1},w_{t+1,i}}$\\
        Update $p_{t+1}\in\Delta_n$ by $p_{t+1,i}\propto \exp \brac{ -\eps_{t+1} \brac{ \sum_{s\in \cI_n} \ell_s + m_{t+1}} }$ with $\eps_{t+1} = \frac{1}{D\sqrt{\Vbar_{\cI_n}}}$
    }
  }
\end{algorithm2e}

\begin{theorem}
    \label{thm:anytime-dynamic}
 Assume that the domain $\cW$ is bounded by $D$.
    For any sequence $u_1,\ldots,u_T$ in $\cW$, Algorithm~\ref{alg:anytime-dynamic} guarantees
    \begin{align*}
        \sumtT \inner{g_t, w_t - u_t} \le O\brac{ \sqrt{ \brac{D^2+DP_T}\Vbar_T }  + \sqrt{DP_T}\hat{M} + P_T\hat{M} + D\hat{M}\brac{\log_+ \tfrac{\Vbar_T}{\gamma^2}} + \gamma(D+P_T)},
    \end{align*}
     where $\Vbar_T=\sumtT\norm{g_t-h_t}^2$, $P_T = \sum_{t=2}^T\norm{u_t - u_{t-1}}$ and $\hat{M} = \max_{t\in[T]}\norm{g_t - h_t}$.
     Moreover, the algorithm runs in $O(d\log t)$ time on iteration $t$.
\end{theorem}

\begin{proof} Denote by $N$ the value of $n$ at the beginning of round $T$. We partition the interval $[T]$ into $N$ segments $\cI_1,\ldots,\cI_N$, and denote by $\cI_n = [s_n,e_n]$, where $s_1=1,e_N=T$.
Note that at the interval $\cI_n$, there are exactly $n$ base learners. Specifically, at round $t$ when $n$ is the current number of base learners, if we find ${\Vbar_{\cI_n}/\gamma^2} > 2^n$ , then we define $e_n=t$ and $s_{n+1}=t+1$, and add one base learner in the next round $t+1$. By definition, we have $\Vbar_{\cI_n} = \sum_{t\in\cI_n}\norm{g_t - h_t}^2$ for all $n\in[N]$.
Moreover, the algorithm ensures that for all $n\in[N-1]$:
\begin{align*}
   \gamma^2 2^n < \Vbar_{\cI_n} = \sum_{t=s_n}^{e_n-1}\norm{g_t - h_t}^2 + \norm{g_{e_n} - h_{e_n}}^2 \le \gamma^2 2^n + \hat{M}^2,
\end{align*}
where we define $\hat{M} = \max_{t\in[T]}\norm{g_t - h_t}$. We also have, without loss of generality, by assuming $N\ge 2$,
\begin{align*}
    \gamma^2 2^{N-1} \le \gamma^2 (2^N - 2) = \gamma^2 \sum_{n=1}^{N-1} 2^n < \sum_{n=1}^{N-1} \Vbar_{\cI_n} \le \sum_{t=1}^{T} \norm{g_t - h_t}^2 = \Vbar_T,
\end{align*}
which implies $N\le 1 + \log_2 (\Vbar_T/\gamma^2)$.\\
Define $\is\in\N$ that $2^{\is}D \le \sqrt{5D^2+12DP_T} < 2^{\is+1} D$. Then decompose the linearized regret by:
\begin{align*}
    \sumtT \inner{g_t, w_t - u_t}
    &= \underbrace{\sum_{n=1}^{N} \sum_{t\in\cI_n} \inner{g_t, w_t - w_{t,\min\{n,\is\}}}}_{\textsc{Meta-Reg}} + \underbrace{\sum_{n=1}^{N} \sum_{t\in\cI_n} \inner{ g_t, w_{t,\min\{n,\is\}} - u_t }}_{\textsc{Base-Reg}}.  
\end{align*}
Consider the base regret $\textsc{Base-Reg}$.
For $n< \is$ and $t\in\cI_n$, the $n$-th base learner starts from $w_{s_n,n}$ and performs one-step variant of Optimistic OGD as in~\pref{lemma:one-step-oogd}, applying which we have:
\begin{align*}
    \sum_{t\in\cI_n} \inner{ g_t, w_{t,n} - u_t }
    &\le \frac{D^2+3DP_{\cI_n}}{2\eta_{e_n+1,n}} + \sum_{t\in\cI_n}\eta_{t+1,n}\norm{g_t - h_t}^2 \\
    &= \brac{ \frac{D^2+3DP_{\cI_n}}{2\cdot 2^n D} + 2\cdot 2^n D  }\sqrt{\bar{V}_{\cI_n}} \\
    &\le \brac{ \frac{D^2+3DP_{\cI_n}}{2\cdot 2^n D} + 2\cdot 2^n D }\brac{\gamma\sqrt{2^n} + \hat{M}} \\
    &\le \frac{\gamma(D+3P_T)}{2\sqrt{2^n}} + \frac{(D+3P_T)\hat{M}}{2\cdot{2^n}} + 2\cdot 2^n D\brac{\gamma\sqrt{2^n} + \hat{M}},
\end{align*}
where in the first line we define $P_{\cI_n} = \sum_{t=s_n}^{e_n-1}\norm{u_t - u_{t+1}}$, in the second line we use the definition $\eta_{t+1,n}=\frac{2^n D}{\sqrt{\Vbar_{t}}}$ and the inequality $\sumtT\frac{a_t}{\sqrt{\sum_{s=1}^{t}a_s}}\le 2\sqrt{\sumtT a_t}$ for any positive sequence $\{a_t\}_{t=1}^T$, then use the condition $\Vbar_{\cI_n}\le \gamma^2 2^n + \hat{M}^2$ in the third line. Then
\begin{align*}
    \sum_{n=1}^{\min\{N,\is-1\}} \sum_{t\in\cI_n} \inner{ g_t, w_{t,n} - u_t }
    &\le O\brac{(D+P_T)(\gamma+\hat{M}) + 4\sum_{n=1}^{\min\{N,\is-1\}}2^n D\brac{\gamma\sqrt{2^n} + \hat{M}}} \\
    &\le O\Bigg((D+P_T)(\gamma+\hat{M}) + 4\sqrt{5D^2+12DP_T} \sum_{n=1}^{\min\{N,\is-1\}}\gamma\sqrt{2^n} \\
    &\qquad\qquad
    + 4\hat{M}\sum_{t=1}^{\min\{N,\is-1\}}2^n D \Bigg)\\
    &\le O\brac{(D+P_T)(\gamma+\hat{M}) + \sqrt{\brac{D^2+DP_T}\Vbar_T} + \sqrt{{D^2 + DP_T}}\hat{M} },
\end{align*}
where we use $\sum_{i=1}^{\infty}\frac{1}{\sqrt{2^i}} = O(1)$ and $\sum_{i=1}^{\infty}\frac{1}{{2^i}} = O(1)$ in the first line, and for $n< \is$ use $2^n D < 2^{\is} D \le \sqrt{5D^2+12DP_T}$ in the second and third line, use $\sum_{n=1}^{N}\gamma\sqrt{2^n}\le O(\gamma\sqrt{2^N}) \le O(\sqrt{\Vbar_T})$ and $\sum_{n=1}^{\is}2^n D \le O(2^{\is}D) \le O(\sqrt{D^2 + DP_T})$ in the third line. Without loss of generality, if $N\ge \is$, then:
\begin{align*}
    \sum_{n=\is}^{N}\sum_{t\in\cI_n}\inner{ g_t, w_{t,\min\{n,\is\}} - u_t }
    &= \sum_{n=\is}^{N}\sum_{t\in\cI_n}\inner{ g_t, w_{\is} - u_t } = \sum_{t=s_{\is}}^T \inner{g_t, w_{\is} - u_t} \\
    &\le \frac{5D^2+12DP_T}{8\eta_{T+1,\is}} + \sum_{t=s_{\is}}^{T}2\eta_{t+1,\is}\norm{g_t - h_t}^2 \\
    &\le \brac{ \frac{5D^2+12DP_T}{8\cdot 2^{\is} D} + 4\cdot 2^{\is} D }\sqrt{\Vbar_T} = O\brac{ \sqrt{ \brac{D^2+DP_T}\Vbar_T } },
\end{align*}
where in the last line we use $2^{\is}D \le \sqrt{5D^2+12DP_T} < 2^{\is+1} D$.
Hence we bound base regret by
\begin{align*}
    \textsc{Base-Reg} \le O\brac{ \sqrt{ \brac{D^2+DP_T}\Vbar_T } + (D+P_T)(\gamma+\hat{M}) + \sqrt{D^2+DP_T}\hat{M} }.
\end{align*}
For meta regret, at the interval $\cI_n = [s_n,e_n]$, we restart an Optimistic Hedge algorithm as in~\pref{lemma:opt-hedge} in domain $\Delta_{n-1}=\Set{a\in\R^n_{\ge 0}:\sum_{i=1}a_i=1}$, with $\ell_{t,i}=\inner{g_t,w_{t,i}}$, $m_{t+1,i} = \inner{h_{t+1},w_{t,i}}$, and $\eps_{t+1} = \frac{1}{D\sqrt{\sum_{s=s_n}^{t}\norm{g_s - h_s}^2}}$, then:
\begin{align*}
    \sum_{t\in\cI_n} \inner{g_t, w_t - w_{t,\min\{n,\is\}}} &= \sum_{t\in\cI_n} \inner{\ell_t, p_t - e_{\min\{n,\is\}}} \\
    &\le \frac{\ln n}{\eps_{e_n}} + \sum_{t\in \cI_n} \inner{\ell_t - m_t, p_t - p_{t+1}} - \sum_{t\in \cI_n} \frac{1}{2\eps_{t-1}}\norm{p_t - p_{t+1}}_1^2 \\
    &\le \frac{\ln n}{\eps_{e_n}} + \sum_{t\in \cI_n} \norm{\ell_t - m_t}_\infty \norm{p_t - p_{t+1}}_1 - \sum_{t\in \cI_n} \frac{1}{2\eps_{t+1}}\norm{p_t - p_{t+1}}_1^2 \\
    &\qquad + \sum_{t\in \cI_n} \brac{\frac{1}{2\eps_{t+1}} - \frac{1}{2\eps_t} }\norm{p_t - p_{t+1}}_1^2  \\
    &\qquad+ \sum_{t\in \cI_n} \brac{\frac{1}{2\eps_{t}} - \frac{1}{2\eps_{t-1}} }\norm{p_t - p_{t+1}}_1^2 \\
    &\le \frac{1 + \ln n}{\eps_{e_n+1}} + \frac{1}{2}\sum_{t\in \cI_n} \eps_{t+1}\norm{\ell_t - m_t}_\infty^2 \\
    &\le \frac{1 + \ln n}{\eps_{e_n+1}} + \frac{D^2}{2}\sum_{t\in \cI_n} \eps_{t+1}\norm{g_t - h_t}^2 \\
    &\le \brac{ 2 + \ln n }D\sqrt{\Vbar_{\cI_n}}.
\end{align*}
Moreover, when $n\in[N-1]$, we have $\Vbar_{\cI_n}\le \gamma^2 2^n + \hat{M}^2$ by definition.
Then we have:
\begin{align*}
    \textsc{Meta-Reg}
    &= \sum_{n=1}^{N} \sum_{t\in\cI_n} \inner{g_t, w_t - w_{t,\min\{n,\is\}}} \le (2+\ln N)D \brac{ \sum_{n=1}^{N-1} \sqrt{\Vbar_{\cI_n}} + \sqrt{\Vbar_{\cI_N}} } \\
    &\le (2+\ln N)D\brac{ \sum_{n=1}^{N-1}\brac{\gamma\sqrt{2^n} + \hat{M}} + \sqrt{\Vbar_T} } = O\brac{ D\sqrt{\Vbar_T} + D\hat{M}\brac{\log_+ \tfrac{\Vbar_T}{\gamma^2}} },
\end{align*}
where we use $N\le 1 + \log_2 (\Vbar_T/\gamma^2)$.
Combining meta-regret and base-regret, we have:
\begin{align*}
    \sumtT \inner{g_t, w_t - u_t} \le O\brac{ \sqrt{ \brac{D^2+DP_T}\Vbar_T }  + \sqrt{DP_T}\hat{M} + P_T\hat{M} + D\hat{M}\brac{\log_+ \tfrac{\Vbar_T}{\gamma^2}} + \gamma(D+P_T)}.
\end{align*}
\end{proof}

We provide for completeness the dynamic regret of optimistic OGD with a time-varying step-size.
The result follows via a mild generalization of the usual mirror-descent argument.
\begin{lemma}[Dynamic Regret of One-step Optimistic OGD]
    \label{lemma:one-step-oogd}
    Assume that the domain $\cW$ is bounded with diameter $D=\sup_{x,y\in\cW}\norm{x-y}$. The one-step  optimistic OGD~\citep{TCS'20:modular-composite} that starts at $w_1\in\cW$ and updates using $\wtpp = \argmin_{\w\in\ww}\inner{\gt-h_{t}+h_{\tpp},\w}+\frac{1}{2\eta_{t+1}}\norm{w - w_t}^2$ with non-increasing sequence $\eta_t>0$ and $h_1=\zeros$, guarantees
    \begin{align*}
        \sumtT\inner{\gt,\wt-\cmp_t}
       &\le 
       \frac{D^2+3DP_T}{2\eta_{T+1}}+\sumtT\eta_\tpp\norm{\gt-h_t}^2
    \end{align*}
    for any sequence $u_1,\ldots,u_T$ in $\cW$, where $P_T=\sum_{t=2}^{T}\norm{u_t - u_{t-1}}$ is the path-length.
\end{lemma}
\begin{proof}
    Let $\psi_t(w)=\frac{1}{2\eta_t}\norm{\w}^2$ for all $t$, and observe that
    $\cD_{\psi_t}(x,y)=\frac{1}{2\eta_t}\norm{x-y}^2$. 
    As usual, we seek to apply the 
    first-order optimality condition 
    $\wtpp=\argmin_{w\in\ww} \inner{\gt-h_t+h_\tpp,\w}+\cD_{\psi_\tpp}(\w,\wt)$,
    so we begin by exposing terms of the form $\inner{\gt-h_t+h_\tpp,\wtpp-\cmp_{\tpp}}$:
    \begin{align}
        \sumtT \inner{\gt,\wt-\cmp_t}
        &=
        \sumtT \inner{\gt-h_t,\wt-\cmp_t}
        +\sumtT \inner{h_t,\wt-\cmp_t}\nonumber\\
        &=
        \sumtT \inner{\gt-h_t,\wtpp-\cmp_\tpp}
        +\sumtT \inner{h_t,\wt-\cmp_t}\nonumber\\
        &\qquad
        +\sumtT \inner{\gt-h_t,\wt-\wtpp}+\inner{\gt-h_t,\cmp_\tpp-\cmp_t}\nonumber\\
        &=
        \underbrace{\sumtT \inner{\gt-h_t+h_\tpp,\wtpp-\cmp_\tpp}
        }_{\encircle{A}}\nonumber\\
        &\qquad
        +\underbrace{\sumtT \inner{h_t,\wt-\cmp_t}-\inner{h_\tpp,\wtpp-\cmp_\tpp}}_{\encircle{B}}\nonumber\\
        &\qquad
        +\sumtT \inner{\gt-h_t,\wt-\wtpp}+\inner{\gt-h_t,\cmp_\tpp-\cmp_t}\nonumber\\
        \label{eq:optimistic-gd-1}
    \end{align}
    where we've defined an arbitrary $\cmp_{T+1}\in\ww$, which we may set to $u_{T}$ without loss of generality.
    Now by the first-order optimality condition
    $\wtpp=\argmin_{w\in\ww}\inner{\gt-h_t+h_\tpp,\w}+\cD_{\psi_\tpp}(w,\wt)$, we have that for any $w\in\ww$ that 
    \begin{align*}
       \inner{\gt-h_t+h_\tpp+\grad\psi_\tpp(\wtpp)-\grad\psi_\tpp(\wt),\wtpp-w} \le 0,
    \end{align*}
    hence, we can bound
    \begin{align*}
        \encircle{A}&=\sumtT\inner{\gt-h_t+h_\tpp,\wtpp-\cmp_\tpp}\\
        &\le 
        \sumtT \inner{\grad\psi_\tpp(\wt)-\grad\psi_\tpp(\wtpp),\wtpp-\cmp_\tpp}\\
        &=
        \sumtT \cD_{\psi_\tpp}(\cmp_\tpp,\wt)-\cD_{\psi_\tpp}(\cmp_\tpp,\wtpp) - \cD_{\psi_\tpp}(\wtpp,\wt)
    \end{align*} 
    where the last line uses the three-point 
    relation for Bregman divergences, $\inner{\grad f(x)-\grad f(x'),x'-u}=\cD_f(u,x)-\cD_f(u,x')-\cD_f(x',x)$.
    Meanwhile, the terms $\encircle{B}$
    telescope to zero:
    \begin{align*}
        \encircle{B}&=\sumtT \inner{h_t,\wt-\cmp_t}-\inner{h_\tpp,\wtpp-\cmp_\tpp}\\
        &=
        \inner{h_1,\w_1-\cmp_1}-\inner{h_{T+1},\w_{T+1}-\cmp_{T+1}} 
        =
        0
    \end{align*}
    where we've used $h_1=\zeros$ and observed that we may set $h_{T+1}=\zeros$ without loss of generality, since the regret does not depend on $h_{T+1}$. Plugging the previous two displays back into~\pref{eq:optimistic-gd-1} and re-arranging terms, we have
   \begin{align}
       \sumtT \inner{\gt,\wt-\cmp_t}
       &=
       \underbrace{\sumtT \cD_{\psi_\tpp}(\cmp_\tpp,\wt)-\cD_{\psi_\tpp}(\cmp_\tpp,\wtpp)}_{\Omega_T}\nonumber\\
       &\qquad
       +\sumtT \inner{\gt-h_t,\wt-\wtpp}-\cD_{\psi_\tpp}(\wtpp,\wt)+\inner{\gt-h_t,\cmp_\tpp-\cmp_t}\nonumber\\
        &\le
        \Omega_T
        +\sumtT \eta_\tpp\norm{\gt-h_t}^2 +\sumtT \frac{1}{2\eta_\tpp}\norm{\cmp_t-\cmp_\tpp}^2\label{eq:optimistic-ogd-2}\\
        &\le
        \Omega_T
        +\sumtT \eta_\tpp\norm{\gt-h_t}^2 +\frac{DP_T}{2\eta_{T+1}},\nonumber
   \end{align} 
   where the second-to-last line applies Fenchel-young inequality twice to bound $\inner{\gt-h_t,\Delta}\le \frac{\eta_\tpp}{2}\norm{\gt-h_t}^2+\frac{1}{2\eta_\tpp}\norm{\Delta}^2$ for $\Delta=\wt-\wtpp$ and for $\Delta=\cmp_\tpp-\cmp_t$, while the last line 
   bounds $\sum_{t=2}^T\frac{1}{2\eta_\tpp}\norm{\cmp_t-\cmp_\tpp}^2\le \frac{DP_T}{2\eta_{T+1}}$ for non-increasing sequence $(\eta_t)_t$ and $\cmp_{T+1}=\cmp_T$.
   Finally, the terms $\Omega_T$ can be bound as
   \begin{align*}
   \Omega_T
       &=
       \sumtT \cD_{\psi_\tpp}(\cmp_\tpp,\wt)-\cD_{\psi_\tpp}(\cmp_\tpp,\wtpp)\\
       &\le
       \sumtT \cD_{\psi_t}(\cmp_\tpp,\wt) - \cD_{\psi_\tpp}(\cmp_\tpp,\wtpp) + \sumtT \cD_{\psi_\tpp-\psi_t}(\cmp_\tpp,\wt)\\
       &\le
       \sumtT \cD_{\psi_t}(\cmp_\tpp,\wt) - \cD_{\psi_\tpp}(\cmp_\tpp,\wtpp) + D^2\sumtT \brac{\frac{1}{2\eta_\tpp}-\frac{1}{2\eta_t}}\\
       &\le 
       \frac{D^2}{2\eta_1}+\sum_{t=2}^T \sbr{\cD_{\psi_t}(\cmp_\tpp,\wt)-\cD_{\psi_t}(\cmp_t,\wt)} + D^2\brac{\frac{1}{2\eta_{T+1}}-\frac{1}{2\eta_1}}\\
       &\le
       \frac{D^2}{2\eta_{T+1}}+\sum_{t=2}^T \frac{1}{\eta_t}\inner{\cmp_\tpp-\wt,\cmp_\tpp-\cmp_t}
       \le 
       \frac{D^2}{2\eta_{T+1}}+\frac{DP_T}{\eta_{T+1}},
    \end{align*}
   where the last line follows via 
   $\half\norm{x-y}^2-\half\norm{x-z}^2\le \inner{x-y,z-y}$ via convexity of $w\mapsto \half\norm{w}^2$.
   Plugging this back into the previous display yields
   \begin{align*}
       \sumtT\inner{\gt,\wt-\cmp_t}
       &\le 
       \frac{D^2+3DP_T}{2\eta_{T+1}}+\sumtT\eta_\tpp\norm{\gt-h_t}^2
   \end{align*}
\end{proof}
We also provide a slightly different version of this result which exposes an additional negative term (the so-called \emph{RVU property}~\citep{NIPS'15:fast-rate-game}).
We were unable to find an explicit statement of this result in the literature (though it is easily derived from, e.g., \citet[Theorem 1]{arxiv'25:Universal-Zhao}), so we provide it here for posterity. 
\begin{corollary}\label{cor:dynamic-RVU}
    Under the same assumptions as~\pref{lemma:one-step-oogd},
    for any sequence $\cmp_1,\ldots,\cmp_T$ it holds that
    \begin{align*}
        R_T(\cmp_1,\ldots,\cmp_T)&\le \frac{D^2+3DP_T}{2\eta_{T+1}}+\sumtT \frac{3\eta_\tpp}{2}\norm{\gt-h_t}^2 - \frac{1}{4\eta_\tpp}\norm{\wt-\wtpp}^2.
    \end{align*}
\end{corollary}
\begin{proof}
The result follows via the same steps as the proof of~\pref{lemma:one-step-oogd}, but 
in line~\pref{eq:optimistic-ogd-2} applies Fenchel-Young inequality as
$\inner{\gt-h_t, \wt-\wtpp}\le \frac{\rho}{2}\norm{\gt-h_t}^2+\frac{1}{2\rho}\norm{\wt-\wtpp}^2$ with $\rho=2\eta_\tpp$,
so that
\begin{align*}
    &\inner{\gt-h_t,\wt-\wtpp}-\cD_{\psi_\tpp}(\wtpp,\wt) \\
    &\qquad\le 
    \eta_\tpp\norm{\gt-h_t}^2 +\frac{1}{4\eta_{\tpp}}\norm{\wt-\wtpp}^2-\frac{1}{2\eta_\tpp}\norm{\wtpp-\wt}^2\\
    &\qquad=
    \eta_\tpp\norm{\gt-h_t}^2 -\frac{1}{4\eta_{\tpp}}\norm{\wt-\wtpp}^2,
\end{align*}
which leads to the stated result by following the same steps thereafter.
\end{proof}

\subsection{Proof of Theorem~\ref{thm:SEA}}
\label{app:proof-SEA}

In this section we prove our main result for the SEA model. The theorem is re-stated below for convenience.
\SEA*
\begin{proof}
The gradient variations can be decomposed as:
\begin{align*}
    \norm{\grad f_t(u_{t-1}) - \grad f_{t-1}(u_{t-1})}^2 
    &=  \|\grad f_t(u_{t-1}) - \grad F_t(u_{t-1}) + \grad F_t(u_{t-1}) - \grad F_{t-1}(u_{t-1}) \\
    &\qquad
    + \grad F_{t-1}(u_{t-1}) - \grad f_{t-1}(u_{t-1})\|^2 \\
    \le{} & 3\norm{\grad F_t(u_{t-1}) - \grad F_{t-1}(u_{t-1})}^2 + 3\norm{\grad f_t(u_{t-1}) - \grad F_t(u_{t-1})}^2 \\
    &\qquad
    + 3\norm{\grad f_{t-1}(u_{t-1}) - \grad F_{t-1}(u_{t-1})}^2.
\end{align*}
By definition~(\pref{eq:def-SEA-sigmas}) we have:
\begin{align*}
    \E\mbr{ \norm{\grad f_t(u_{t-1}) - \grad f_{t-1}(u_{t-1})}^2 } \le 3\Sigma_t^2(u_{t-1}) + 3\sigma_t^2(u_{t-1}) + 3\sigma_{t-1}^2(u_{t-1}).
\end{align*}
Therefore, applying~\pref{thm:dynamic} and taking expectations, an application of Jensen's inequality yields
\begin{align*}
    &\E\mbr{\sqrt{(M^2 + MP_T)V_T(\cmp_{1:T})\brac{\log_+ \frac{MG\sqrt{T}}{\epsilon \gamma}}}} \\
    &\qquad\le{} \sqrt{\E\mbr{(M^2 + MP_T)V_T(\cmp_{1:T})\brac{\log_+ \frac{MG\sqrt{T}}{\epsilon \gamma}}}} \\
    &\qquad={}  \sqrt{(M^2 + MP_T)\brac{\log_+ \frac{MG\sqrt{T}}{\epsilon \gamma}} \E\mbr{V_T(\cmp_{1:T})}} \\
    &\qquad\le{}  \sqrt{(M^2 + MP_T)\brac{\log_+ \frac{MG\sqrt{T}}{\epsilon \gamma}} (3\Sigma_{1:T}^2 + 3\sigma_{1:t}^2)},
\end{align*}
where we define $\Sigma_{1:T}^2 \define \sum_{t=2}^T\Sigma_t^2(u_{t-1})$, and $\sigma_{1:t}^2 \define \sum_{t=2}^T(\sigma_t^2(u_{t-1}) + \sigma_{t-1}^2(u_{t-1}))$.
\end{proof}

\section{Supporting Lemmas}
\label{app:lemmas}

In this section, we collect various supporting lemmas that are used throughout. 
The results in this section are mostly used to simplify calculations or restate existing results for completeness.

The following lemma provides a formula that relates the empirical gradient variation to $V_T(\cmp_{1:T})$.
\begin{restatable}{lemma}{TediusCalculations}\label{lemma:tedius-calculations}
  Let $f_{1},\ldots,f_{T}$ be an arbitrary sequence of $L$-smooth functions.
  Then
  \begin{align*}
    &\sumtT \norm{\grad f_{t}(\wt)-\grad f_{\tmm}(\wtmm)}^{2}
    \le 16 L\sumtT \cD_{\ft}(\cmp_t,\wt)\\
    &\qquad+
    \Min{\norm{\nabla f_1(w_1)}^2 + 4V_T(\cmp_{1:T}) + 4L^2P_T^{\norm{\cdot}^2}(\cmp_{1:T}), 16LF_T(u_{1:T}) } 
  \end{align*}
  where we denote
  $V_{T}(\cmp_{1:T})\define\sum_{t=2}^T \norm{\grad f_{t}(\cmp_{t-1})-\grad f_{\tmm}(\cmp_{t-1})}^{2}$,
  $F_T(\cmp_{1:T})\define\sumtT (\ft(\cmp_t)-\inf_{w\in\R^d}\ft(w))$, 
  and $P_{T}^{\norm{\cdot}^{2}}(\cmp_{1:T})\define\sum_{t=2}^T\norm{\cmp_{t}-\cmp_{\tmm}}^{2}$,
\end{restatable}
\begin{proof}
  For brevity denote $\gt(w)=\grad f_{t}(\w)$ and define $f_0(\cdot)\define 0$.
  \paragraph{Bounding via Gradient Variation.}
  For all $t\ge 2$, 
  \begin{align*}
    \norm{\gt(\wt)-\gtmm(\wtmm)}
    &=
      \|\gt(\wt)-\gt(\cmp_{t})+\gt(\cmp_{t})-\gt(\cmp_{t-1}) \\
      &\qquad+ \gt(\cmp_{t-1}) -\gtmm(\cmp_{\tmm})+\gtmm(\cmp_{\tmm})-\gtmm(\wtmm)\|\nonumber\\
    &\le
      \norm{\gt(\wt)-\gt(\cmp_{t})}+\norm{\gtmm(\wtmm)-\gtmm(\cmp_{\tmm})} \\
      &\qquad + \norm{\gt(\cmp_{t})-\gt(\cmp_{t-1})} + \norm{\gt(\cmp_{t-1})-\gtmm(\cmp_{\tmm})}.
  \end{align*}
  Then using Cauchy-Schwarz inequality
  $(\sum_{i=1}^{n}a_{i})^{2}\le \brac{\sqrt{n \sum_{i=1}^{n}a_{i}^{2}}}^{2}=n\sum_{i=1}^{n}a_{i}^{2}$,
  we can bound $\sumtT\norm{\gt(\wt)-\gtmm(\wtmm)}^{2}$ by
  \begin{align}
    &\sumtT\norm{\gt(\wt)-\gtmm(\wtmm)}^{2}\notag\\
    \le{}& \norm{g_1(w_1)}^2 +
    4\sum_{t=2}^T\Bigg[\norm{\gt(\wt)-\gt(\cmp_{t})}^{2}+\norm{\gtmm(\wtmm)-\gtmm(\cmp_{\tmm})}^{2}\notag\\
    &\qquad\qquad\qquad\qquad + \norm{\gt(\cmp_{t})-\gt(\cmp_{t-1})}^{2}+\norm{\gt(\cmp_{t-1})-\gtmm(\cmp_{\tmm})}^{2}\Bigg]\notag\\
    \le{}& \norm{g_1(w_1)}^2 +
    16L\sumtT \cD_{f_t}(\cmp_t, w_t) + 4\sum_{t=2}^T \norm{\gt(\cmp_{t})-\gt(\cmp_{t-1})}^{2} \notag\\
    &\qquad+ 4\sum_{t=2}^T \norm{\gt(\cmp_{t-1})-\gtmm(\cmp_{\tmm})}^{2}\notag\\
    \le{}& \norm{g_1(w_1)}^2 +
    16L\sumtT \cD_{f_t}(\cmp_t, w_t) + 4L^2\sum_{t=2}^T \norm{\cmp_{t}-\cmp_{\tmm}}^{2} + 4 V_T(u_{1:T})\notag\\
    ={} & \norm{g_1(w_1)}^2 + 16L\sumtT \cD_{f_t}(\cmp_t, w_t) + 4L^2 P_T^{\norm{\cdot}^2}(u_{1:T}) + 4 V_T(u_{1:T}).\label{eq:tedius-calculations-V}
  \end{align}
  where the second inequality uses $\norm{\grad f_t(x)-\grad f_t(y)}^2\le 2L \cD_{\ft}(y,x)$ for $L$-smooth functions (\pref{lemma:smoothness-bregman}),
  the last inequality uses $\norm{\nabla f_t(x) - \nabla f_t(y)}\le L\norm{x - y}$, and we've defined the $V_T(u_{1:T}) \define \sum_{t=2}^{T}\norm{\nabla f_t(u_{t-1}) - \nabla f_{t-1}(u_{t-1})}^2$ and $P_T^{\norm{\cdot}^2}(u_{1:T})\define\sum_{t=2}^T\norm{u_t - u_{t-1}}^2$

  \paragraph{Bounding via Comparator Loss.} Again, for all $t\ge 2$:
  \begin{align*}
    \norm{\gt(\wt)-\gtmm(\wtmm)} = {}& \|\gt(\wt)-\gt(\cmp_{t})+\gt(\cmp_{t}) -\gtmm(\cmp_{\tmm})+\gtmm(\cmp_{\tmm})-\gtmm(\wtmm)\|\nonumber\\
    \le {}&  \norm{\gt(\wt)-\gt(\cmp_{t})}+\norm{\gtmm(\wtmm)-\gtmm(\cmp_{\tmm})} + \norm{\gt(\cmp_{t})} + \norm{\gtmm(\cmp_{\tmm})}.
  \end{align*}
  And for $t=1$, we bound $\norm{\g_1(\w_1)}\le
      \norm{\g_1(\w_1)-\g_1(\cmp_1)} + \norm{\g_1(\cmp_1)}$.
  Hence, 
  we can bound 
  \begin{align}
      \sumtT\norm{\gt(\wt)-\gtmm(\wtmm)}^{2}
      &\le
      \norm{g_1(w_1)}^2
      + 4\sumtT\Big[\norm{\gt(\cmp_{t})}^2 + \norm{\gtmm(\cmp_{\tmm})}^2 \Big] \notag\\
      &\qquad
      + 4\sum_{t=2}^{T} \Big[ \norm{\gt(\wt)-\gt(\cmp_{t})}^2 +\norm{\gtmm(\wtmm)-\gtmm(\cmp_{\tmm})}^2\Big] \notag\\
      &\le  8\sumtT \norm{g_t(u_t)}^2 +16L\sumtT \cD_{f_t}(u_t, w_t) \notag\\
      &\le  16L\sumtT \brac{f_t(u_t) - \inf_{w\in\R^d }f_t(w)}+16L\sumtT \cD_{f_t}(u_t, w_t)  \notag\\
      &= 
       16L F_T(\cmp_{1:T})+16L\sumtT \cD_{f_t}(u_t, w_t),\label{eq:tedius-calculations-F}
  \end{align}
  where the second inequality uses $\norm{\grad f_t(x)-\grad f_t(y)}^2\le 2L \cD_{\ft}(y,x)$ for $L$-smooth functions (\pref{lemma:smoothness-bregman}).
  The last inequality uses the self-bounding property of smooth functions (\pref{lemma:self-bounding}), and the last line defines $F_T(u_{1:T})\define \sumtT (f_t(u_t) - \inf_{w\in\R^d }f_t(w))$.
  Combining~\pref{eq:tedius-calculations-V} and~\pref{eq:tedius-calculations-F} concludes the proof.
\end{proof}

The following lemma provides a crucial bound on the difference between the gradients of smooth functions in terms of the Bregman divergence. 
\begin{lemma}[Theorem 2.1.5 of~\citet{book'18:Nesterov-OPT}]
\label{lemma:smoothness-bregman}
Let $f:\R^d\to\R$ be a $L$-smooth function, then for all $x,y\in\R^d$,
\begin{equation*}
    \norm{\grad f(x) - \grad f(y)}^2 \le 2L\cD_f(x,y).
\end{equation*}
\end{lemma}

The following lemma provides the well-known self-bounding property of smooth functions (see, e.g., \citet{levy2017online}).
\begin{lemma}
\label{lemma:self-bounding}
Let $f:\R^d\to\R$ be a $L$-smooth function. Then for all $x\in\R^d$,
\begin{align*}
    \norm{\grad f(x)}^2\le 2L\left(f(x)-\inf_{z\in\R^d} f(z)\right).
\end{align*}
\end{lemma}

\begin{proof}
By smoothness, we have, for all $x\in\R^d$:
\begin{align*}
    f(x - \frac{1}{L}\nabla f(x))
    \le f(x) + \inner{\nabla f(x), -\frac{1}{L}\nabla f(x)} + \frac{1}{2L}\norm{\nabla f(x)}^2 = f(x) -  \frac{1}{2L}\norm{\nabla f(x)}^2.
\end{align*}
Rearranging, we have
\begin{align*}
    \norm{\nabla f(x)}^2 \le 2L( f(x) -  f(x - \frac{1}{L}\nabla f(x)) ) \le 2L (f(x) - \inf_{z\in\R^d} f(z)).
\end{align*}
\end{proof}

\noindent Finally, we borrow the following lemmas from existing works, included here for completeness.
\begin{lemma}[Adapted from Lemma 21 of~\citet{cutkosky2024fully}]
  \label{lemma:tuning-sum}
Suppose $g_1,\ldots,g_T$ and $0<h_1\le h_2\le\ldots \le h_T$ are such that $\|g_t\|\le h_t$ for all $t$. Define $\Vbar_t = 4h_t^2 + \sum_{i=1}^{t-1}\|g_i\|^2$, $B_t=4+\sum_{i=1}^{t-1}\|g_i\|^2/h_i^2$, and $\alpha_t=\frac{\epsilon}{\sqrt{B_t} \log^2(B_t)}$, then:
\begin{align*}
  \sumtT \frac{\alpha_t \|g_t\|^2}{\sqrt{\Vbar_t}} \le 4\epsilon h_T.
\end{align*}
\end{lemma}

\begin{lemma}[Theorem~7.36 of~\citet{orabona2019modern}]
    \label{lemma:opt-hedge}
    The Optimistic Hedge algorithm, which starts from $p_1=\frac{1}{N}\mathbf{1}_N$ and updates $p_{t+1}\in\Delta_N$ by $p_{t+1,i}\propto \exp(-\eps_{t+1}(\sum_{s=1}^t\ell_t + m_{t+1}))$ with a time-varying and non-increasing learning rate $\eps_t > 0$,  ensures that for any expert $i\in[N]$:
    \begin{align*}
        \sumtT \inner{\ell_t, p_t - e_i} \le \frac{\ln N}{\eps_T} + \sumtT\inner{\ell_t - m_t, p_t - p_{t+1}} - \sumtT \frac{1}{2\eps_{t-1}}\norm{p_t - p_{t+1}}_1^2.
    \end{align*}
\end{lemma}

\end{document}